\documentclass[final,5p,times,twocolumn]{elsarticle}

\usepackage{amsmath}
\usepackage{color}
\usepackage{bm}
\usepackage{amsfonts}
\usepackage{amsthm}
\usepackage{algorithmic}
\usepackage{amssymb}
\usepackage{tikz}
\usepackage[
  locale = DE % if you want a comma as decimal separator instead of a dot
]{siunitx}
\usepackage[version = 3]{mhchem} % typesetting chemical formulae using \ce{}
\usepackage{booktabs}
\usepackage{lipsum}

\usepackage{multicol}

\usepackage{subcaption}

\usepackage{url}
\usepackage[justification=raggedright]{caption}
\usepackage{graphicx}
\usepackage{microtype}

\newtheorem{theorem}{Theorem}
\newtheorem{proposition}{Proposition}
\newtheorem{corollary}{Corollary}
\newtheorem{lemma}{Lemma}
\newtheorem{example}{Example}
\newtheorem*{remark}{Remark}

\journal{xxx}

\begin{document}

\begin{frontmatter}


\title{Improved weight initialization for deep and narrow feedforward neural network}



\author[label1]{Hyunwoo Lee}
\ead{lhw908@knu.ac.kr}

\author[label2]{Yunho Kim}
\ead{yunhokim@unist.ac.kr}

\author[label3]{Seung Yeop Yang}
\ead{seungyeop.yang@knu.ac.kr}

\author[label1]{Hayoung Choi\corref{cor1}}
\cortext[cor1]{Corresponding author.}
\ead{hayoung.choi@knu.ac.kr}

\address[label1]{Department of Mathematics, Kyungpook National University, Daegu 41566, Republic of Korea.}
\address[label2]{Department of Mathematical Sciences, Ulsan National Institute of Science and Technology, Ulsan 44919, Republic of Korea.}
\address[label3]{KNU LAMP Research Center, KNU Institute of Basic Sciences, Department of Mathematics, Kyungpook National University, Daegu, 41566, Republic of Korea}

\begin{abstract}
Appropriate weight initialization settings, along with the ReLU activation function, have become cornerstones of modern deep learning, enabling the training and deployment of highly effective and efficient neural network models across diverse areas of artificial intelligence. 
The problem of \textquotedblleft dying ReLU," where ReLU neurons become inactive and yield zero output, presents a significant challenge in the training of deep neural networks with ReLU activation function. 
Theoretical research and various methods have been introduced to address the problem. 
However, even with these methods and research, training remains challenging for extremely deep and narrow feedforward networks with ReLU activation function. In this paper, we propose a novel weight initialization method to address this issue.
We establish several properties of our initial weight matrix and demonstrate how these properties enable the effective propagation of signal vectors. 
Through a series of experiments and comparisons with existing methods, we demonstrate the effectiveness of the novel initialization method.
\end{abstract}


\end{frontmatter}

\section{Introduction}
\label{sec:intro}

Training neural networks have enabled dramatic advances across a wide variety of domains, notably image recognition~\cite{ImageNEt}, natural language processing~\cite{NLP} 
and generative models~\cite{GAN}.
Numerous well-known neural networks belong to the family of feedforward neural networks~(FFNNs), which are used for input-output mapping. 
Traditionally, FFNN connection weights are optimized by the back-propagation algorithm~\cite{Backpropagation}. 
In the early stages of research on FFNNs, the networks with one or a limited number of hidden layers, which are now referred to as shallow networks, were common.
Consequently, extensive research was conducted to understand their properties.
Notably, these networks have been demonstrated to serve as general function approximators~\cite{approximator2, approximator1, approximator3, approximator4}. Deeper networks, with their multiple hidden layers, have shown enhanced performance in tasks requiring high levels of pattern recognition, such as image and speech analysis~\cite{deeper}. However, as the depth of FFNNs increases, the problem of the vanishing gradient becomes more pronounced~\cite{vanishing_gradient}.
This occurs because the network weights receive updates based on the derivative of the error during training. In certain situations, the gradient becomes extremely small, making 
it almost impossible for weights to change, and in severe cases, it can halt the training 
process altogether.

The rectified linear unit~(ReLU) is one of the most widely-used activation functions in the field of deep learning due to its superior training performance compared to other activation functions~\cite{relu4}.The phenomenon known as \textquotedblleft dying ReLU" is a type of vanishing gradient issue when ReLU neurons become inactive and an output of 0 for any input~\cite{relu1}. It has been widely recognized as a major obstacle to training deep neural networks with ReLU activation function~\cite{relu2, relu3}. To address this issue, several methods have been introduced. These methods can be broadly classified into three general approaches. The first approach is to change network architectures, including the activation functions~\cite{activation3, activation4, activation2, activation}. Another approach involves various normalization techniques~\cite{normal1, normal3, normal2}. The third approach specifically is to study the weights and biases initialization with fixed network architectures~\cite{Xavier, He, Orthogonal}. The third approach is the topic of our research in this paper.

Numerous papers have discussed various weight initialization methods for neural networks and emphasized their importance~\cite{important_weight}. 
The most popular weight initializations are Xavier initialization~\cite{Xavier} and Kaiming initialization~\cite{He}.
Both methods adjust the variance of the initial weight matrix to prevent the vanishing/exploding problem, enabling deeper networks to be trained.
Saxe et al.~\cite{Orthogonal} discuss an orthogonal initialization method based on an orthonormal basis.
ZerO initialization~\cite{zero} which is fully deterministic initialization has benefits in training extremely deep neural networks without batch normalization. 
ZerO initialization utilizes Hadamard transforms to break the training degeneracy. 
For more details, see the review paper~\cite{review} and references therein.

Although deep and wide networks are popular and the most successful in practice, deep and wide networks need high computational costs to train a huge number of parameters.
On the other hand, deep and narrow networks also play important roles theoretically and practically.
As demonstrated in \cite{reluint3}, deep and narrow ReLU networks are essential when creating finite element basis functions.
This application highlights the use of deep ReLU networks in finite element methods for solving partial differential equations.
Additionally, various theoretical studies~\cite{reluint8, reluint4,reluint7, reluint5, reluint6} exploring the 
the expressive power of ReLU networks heavily depends on deep and narrow networks for approximating polynomials through sparse concatenations.

Weight initialization methods have been developed to prevent the dying ReLU problems in deep and narrow FFNNs with the ReLU activation function.
Lu et al.~\cite{reluinit1} provided rigorous proof that as the depth of a deep FFNNs with ReLU activation function approaches infinity, it will eventually become inactive with a certain probability.
Then they propose a randomized asymmetric initialization~(RAI) designed to prevent the dying ReLU problem effectively.
Burkholz et al.~\cite{reluinit2} calculated the precise joint signal output distribution for FFNNs with Gaussian weights and biases, without relying on mean field assumptions, and analyzed deviations from the mean field results. They further discussed the limitations of the standard initialization method, such as its lack of dynamical isometry, and
proposed a simple alternative weight initialization method, namely, the Gaussian submatrix initialization~(GSM).
These studies have improved training performance in deep and narrow feedforward ReLU networks. Despite these advancements, our experiments show that existing methods did not perform well in extremely deep or narrow scenarios. To overcome the problem, this article proposes a novel weight initialization method for FFNNs with ReLU activation functions.
The proposed weight initialization has several properties such as orthogonality, positive entry predominance, and fully deterministic.
Furthermore, due to the properties of the proposed initial weight matrix, it effectively transmits signals even in deep and narrow FFNNs with ReLU activation.

We empirically benchmarked our proposed weight initialization method 
on MNIST~\cite{mnist} and Fashion MNIST datasets comparing to previous weight initialization methods 
such as Xavier~\cite{Xavier}, He~\cite{He}, Orthogonal~\cite{Orthogonal}, Identity, ZerO~\cite{zero}, RAI~\cite{reluinit1}, and GSM~\cite{reluinit2}. Initially, we applied our proposed weight initialization method to various dataset sizes of FFNN models using ReLU activation functions. Our method significantly improves validation accuracy in the models with no hidden layers or in narrower networks with fewer nodes, clearly outperforming other initialization methods. Moreover, various computational experiments show that the proposed method holds depth independence, width independence, and activation function independence. For depth independence,  experiments were conducted on both the MNIST and Fashion MNIST datasets and tabular datasets like the Wine Quality dataset~\cite{wine}, and the Iris dataset~\cite{iris}. It demonstrated that the proposed method performs well for depth independence, excelling in training deep feedforward neural networks across different numbers of layers. Also, it was shown that our method holds width independence, effectively training networks with various numbers of nodes per layer. It achieved robust validation accuracy and rapid convergence, even in network configurations that traditionally challenge other weight initialization methods. Moreover, our method demonstrated independence from activation functions in the ReLU family. The preceding experiments underscore that the proposed initialization method is more independent of network architecture. 

\subsection*{A. Contributions}
In this paper, we propose a novel weight initialization method in extremely deep and narrow feedforward neural networks (FFNNs) with a rectified linear unit (ReLU) activation function. The main contributions of this paper are summarized as follows.

\begin{itemize}
    \item We propose a novel weight initialization method that prevents the dying ReLU problem in extremely deep and narrow FFNNs with ReLU activation function.
    \item We analyze the properties of the proposed initial weight matrix. 
    We demonstrated orthogonality and the absolute value of column sum of $\mathbf{Q}^{\epsilon}$ is less or equal to $\epsilon$. Furthermore, we show that $\mathbf{W}^{\epsilon}$ with a constant row(or column) sum. We also show that $\mathbf{W}^{\epsilon}\mathbf{x}$ has more positive entries.
    \item We conducted experiments applying our proposed method and existing methods in various scenarios.
\end{itemize}

\subsection*{B. Organization and Notations}
The remainder of this paper is organized as follows. 
In Section~\ref{sec:methodology}, we present existing weight initialization methods and introduce our proposed weight initialization method. 
Next, various properties of the proposed initial weight matrix are provided in Section~\ref{sec:property}.
Section~\ref{sec:experiments} presents simulation results.
Finally, conclusions are drawn in Section~\ref{sec:conclusion}.

\smallskip

{\bf{Notations}}: Let $\mathbb{R}$ be the set of real numbers and $\mathbb{R}_{+}$ be the set of nonnegative real numbers. 
The standard inner product of two vectors $\mathbf{u}$ and $\mathbf{v}$ is denoted by $\langle \mathbf{u}, \mathbf{v} \rangle$, and $\|\mathbf{v}\|$ denotes the Euclidean norm. Denote the $m\times n$ matrix whose all entries are ones as $\mathbf{J}_{m\times n}$
and denote the $m\times n$ matrix with ones on the main diagonal and zeros elsewhere as $\mathbf{I}_{m\times n}$.
For $m=n$ we simply denote $\mathbf{I}_m$ and $\mathbf{J}_m$ instead of $\mathbf{I}_{m\times m}$ and $\mathbf{J}_{m \times m}$, respectively.
Denote $\mathbf{1}_{m} = \left[1~ 1~ \cdots~ 1 \right]^T \in \mathbb{R}^{m}$.
However, if the size is clear from context, we will drop $m$ from our notation for brevity.
 $\mathbf{e}_j$ ($j=1,\ldots,m$) denotes the vector in $\mathbb{R}^m$ with a $1$ in the $j$-th coordinate and $0$'s elsewhere. $\mathcal{O}(\cdot)$ represents the big $O$ notation.

\bigskip

\section{Methodology}\label{sec:methodology}
Before introducing our proposed weight initialization method, we briefly give basic concepts and prior work.
\subsection{Basic Conceptions}
Let $K$ pairs of training samples $\{(\mathbf{x}_i, \mathbf{y}_i)\}_{i=1}^K$, where $\mathbf{x}_i\in\mathbb{R}^{N_x}$ is training input and $\mathbf{y}_i\in\mathbb{R}^{N_y}$ is its corresponding output. 
Here, $N_x$ and $N_y$ are the number of nodes in the input layer and output layer, respectively.
The result $\mathbf{y}_i$ will be a vector with continuous values in the case of regression problems, a binary one-hot vector for classification problems, and so forth.
An FFNN with $L$ layers performs cascaded computations of
\begin{equation*}
    \mathbf{x}^\ell = f(\mathbf{z}^\ell) = f(\mathbf{W}^{\ell}\mathbf{x}^{{\ell}-1}+\mathbf{b}^{\ell})\in\mathbb{R}^{N_{\ell}} \quad \text{for  all }\ell=1,\ldots,L,
\end{equation*}
where $\mathbf{x}^{\ell-1}\in\mathbb{R}^{N_{\ell-1}}$ is the input feature of $\ell$-th layer,
$\mathbf{W}^{\ell}\in\mathbb{R}^{N_{\ell}\times N_{\ell-1}}$ is the weight matrix,
$\mathbf{b}^{\ell}\in\mathbb{R}^{N_{\ell}}$ is the bias vector for each $\ell=1,\ldots,L$, 
and $f(\cdot)$ is an element-wise activation function.
To gain good estimation of $\mathbf{y}$ for any test sample $\mathbf{x}$, FFNNs optimization aims to find optimal solutions of network parameters $\Theta = \{ \mathbf{W}^{\ell}, \mathbf{b}^{\ell}\}_{\ell=1}^L$.
In other words, training is the process of solving the following equation:
\begin{equation*}
    \min_{\Theta} \mathcal{L}(\{(\mathbf{x}_i, \mathbf{y}_i)\}_{i=1}^K;\Theta),
\end{equation*}
where $\mathcal{L}$ is a training loss function.

The network parameters $\Theta = \{ \mathbf{W}^{\ell}, \mathbf{b}^{\ell}\}_{\ell=1}^L$ are usually optimized using gradient descent. The gradient descent updates the network parameter with an initialization as follows:
for each $t=0,1,2, \ldots$,
\begin{align*}
    \mathbf{W}^{\ell}(t+1)
    &=\mathbf{W}^{\ell}(t)-\eta \frac{\partial \mathcal{L}}{\partial \mathbf{W}^{\ell}(t)}\quad (\ell=1,\ldots,L),\\
    \mathbf{b}^{\ell}(t+1)
    &=\mathbf{b}^{\ell}(t)-\eta \frac{\partial \mathcal{L}}{\partial \mathbf{b}^{\ell}(t)} \quad (\ell=1,\ldots,L),
\end{align*}
where $\eta>0$ is the learning rate.
There exist variants of gradient descents such as stochastic gradient descents~(SGD), ADAM~\cite{adam}, AdaGrad~\cite{adagrad}, and so on. 

\medskip

\subsection{Prior Work}
Weight initialization plays a critical role in training neural networks, significantly influencing model convergence and learning performance~\cite{review}. Selecting an appropriate initialization method is vital for improving a model's efficiency and performance. These initialization methods affect the convergence rate and training stability of learning algorithms like gradient descent.
Among well-known approaches are the Xavier and He initialization methods. These research efforts involve scaling the initial weights to maintain the variance of input and output layers or control the variance of the output layer to a desired value. They also focus on preserving the variance of gradients during training, all of which contribute to more effective and stable neural network training.
However, choosing the right variance for weight initialization becomes increasingly complex, particularly with a growing number of layers. 
Addressing these challenges, Zhao et al.~\cite{zero} introduced ZerO, a fully deterministic initialization method. The method initializes network weights to either zeros or ones. This novel method is grounded in identity and Hadamard transforms, serving as a replacement for the traditional random weight initialization. ZerO offers numerous advantages, including the ability to train exceptionally deep networks without requiring batch normalization.
The orthogonal initialization method employs an orthogonal matrix for weight initialization~\cite{Orthogonal}. 
The method ensures that the singular values of the input-output Jacobian are approximately equal to $1$. This condition, known as dynamical isometry, allows for consistent learning times that are not dependent on the depth of the neural networks.
Although deep and wide networks are effective and popular, they incur high computational costs from their extensive parameters. Conversely, deep and narrow networks hold substantial theoretical and practical significance. They are essential in creating finite element basis functions, particularly in applications like solving partial differential equations, as shown in \cite{reluint3}. Moreover, a variety of theoretical investigations~\cite{reluint8,reluint4, reluint7, reluint5, reluint6} into the expressive power of ReLU networks rely heavily on deep and narrow networks to approximate polynomials efficiently through sparse concatenations.
Yet, the \textquotedblleft dying ReLU\textquotedblright problem remains a significant obstacle in training deep and narrow FFNNs.
Lu et al.~\cite{reluinit1} rigorously proved that as the depth of deep FFNNs with ReLU activation function tends toward infinity, it will eventually become inactive with a certain probability. They also introduced a randomized asymmetric initialization method~(RAI) to effectively address the dying ReLU problem. Burkholz et al.~\cite{reluinit2}, on the other hand, calculated the precise joint signal output distribution for FFNNs with Gaussian weights and biases. Without relying on mean-field assumptions, they analyzed deviations from the mean-field results and discussed the limitations of the standard weight initialization method. They proposed an alternative weight initialization approach known as Gaussian submatrix initialization~(GSM).
However, the methods proposed so far have shown limited effectiveness in extremely deep and narrow FFNNs. To address this issue, we propose a new weight initialization method.

\medskip

\subsection{Proposed Weight Initialization Method}
Our proposed weight initialization method can be characterized by key properties: orthogonality, positive entry predominance, and fully deterministic. Proposition~\ref{prop:orthogonal} establishes that the proposed initial weight matrix is orthogonal. Orthogonal weight initialization, extensively studied both theoretically and empirically, has been shown to accelerate convergence in deep linear networks through the attainment of dynamical isometry~\cite{Orthogonal3,Orthogonal2,Orthogonal}. Our method demonstrates in Proposition~\ref{prop:csum} that the initial weight matrix's entry sum of each column (resp. row) vector is almost the same. Building on this, Corollary~\ref{cor1} establishes that each $\mathbf{Wx}$ has more positive entries, thereby preventing the dying ReLU problem in deep networks. Finally, the proposed weight initialization is fully deterministic, thus it is not dependent on randomness. 

\smallskip

To construct a proper initial weight matrix, we find a matrix $\mathbf{W}\in\mathbb{R}^{m\times n}$ satisfying the following conditions:
\begin{itemize}
\item[$($i$)$] The set of all column vectors of $\mathbf{W}$ is orthonormal; 
\item[$($ii$)$] $\mathbf{W}\mathbf{x}$ has more positive entries for all $\mathbf{x}\in \mathbb{R}_{+}^{n}$;
\item[$($iii$)$] $\mathbf{W}$ is a fully deterministic matrix.     
\end{itemize}

\smallskip

To obtain such a matrix we first define $\mathbf{Q}_{m\times m}^{\epsilon}$ by the orthogonal matrix of a QR decomposition of 
\begin{equation*}
\mathbf{J}^{\epsilon}:=
\mathbf{J} + \epsilon \mathbf{I}=
    \begin{bmatrix}
        1+\epsilon & 1 & \cdots & 1\\
        1 & 1+\epsilon  & \cdots & 1 \\
        \vdots & \vdots & \vdots & \vdots  \\
        1 & 1 & \cdots &  1+\epsilon
    \end{bmatrix}_{m\times m},
\end{equation*}
where $\epsilon>0$ is a sufficiently small.

To initialize the weights of the neural networks we propose that
\begin{equation}\label{eq:proposedW1}
\mathbf{W}^{\epsilon}_{m\times n}=  \left( \mathbf{Q}_{m\times m}^{\epsilon}  \right) \mathbf{I}_{m\times n} \left( \mathbf{Q}_{n\times n}^{\epsilon} \right)^T.
\end{equation}
It is noteworthy that 
$\mathbf{W}^{\epsilon}_{m\times n}$ can be expressed as 
\begin{equation}\label{eq:outerprod1}
    \mathbf{W}_{m\times n}^{\epsilon} = \mathbf{q}_1 \hat{\mathbf{q}}_1^T + \mathbf{q}_2 \hat{\mathbf{q}}_2^T + \cdots + \mathbf{q}_{s} \hat{\mathbf{q}}_s^T,
\end{equation}
where $s=\min\{m,n\}$, and $\mathbf{q}_1,\ldots,\mathbf{q}_m$ are the column vectors of $\mathbf{Q}_{m\times m}^{\epsilon}$ and $\hat{\mathbf{q}}_1,\ldots,\hat{\mathbf{q}}_n$ are the column vectors of $\mathbf{Q}_{n\times n}^{\epsilon}$.
Note that $\mathbf{q}_1,\ldots,\mathbf{q}_m$ are orthonormal vectors in $\mathbb{R}^m$ and $\hat{\mathbf{q}}_1,\ldots,\hat{\mathbf{q}}_m$ are orthonormal vectors in $\mathbb{R}^n$.
That is, two sets of column vectors are constructed very similarly, but they are defined in different dimensional vector spaces for $m\neq n$. Moreover, $\mathbf{q}_i \hat{\mathbf{q}}_i^T$ is a rank-one matrix for all $i$.

\begin{remark}
    The matrix $\mathbf{J}^{\epsilon}=\mathbf{J}_m + \epsilon \mathbf{I}_m$ is positive definite, specifically, the eigenvalues consist of $\lambda_1=m+\epsilon$ and $\lambda_2=\epsilon$ (multiplicity is $m-1$).
The corresponding eigenvector of $\lambda_1$ is $\mathbf{1}$ and the corresponding eigenvectors of $\lambda_2 $ are the set of independent vectors $\{\mathbf{v}_2,\ldots,\mathbf{v}_{m} \}$ such that $\mathbf{1} \perp \mathbf{v}_i$ for all $i=2,\ldots, m$. 
For more details on the matrix $\mathbf{J}^{\epsilon}$, see the paper~\cite{mat1, mat2}. 
\end{remark}

We first give the proposed initial weight matrix $\mathbf{W}^{\epsilon}_{m\times n}$ for small values $m,n$.
\begin{example}
For $\epsilon=0.01$ 
the proposed initial weight matrix is computed approximately as follows. 
\begin{align*}
\mathbf{W}^{\epsilon}_{3\times 2}
&=
    \begin{bmatrix}
        -0.0829 & 0.9097\\
        0.9081 & -0.0993\\
        0.4106 & 0.4032
    \end{bmatrix},\\
\mathbf{W}^{\epsilon}_{4\times 3}
&=
    \begin{bmatrix}
 0.6241 &  -0.3762   &  0.6213\\
   -0.3754  &   0.6242   &  0.6217\\
    0.6213  &  0.6209 &  -0.3816\\
    0.2890 &    0.2887 &   0.2862 
\end{bmatrix}.
\end{align*}
\end{example}

\begin{example}
For $\epsilon_1=0.0001$ and $\epsilon_2=0.1$
the proposed initial weight matrix $\mathbf{W}^{\epsilon}_{m\times n}$ is computed approximately as follows. 
\begin{align*}
\mathbf{W}^{\epsilon_1}_{8\times 5}
&=
    \begin{bmatrix}
          \textcolor{red}{0.8581}  & \textcolor{cyan}{-0.1419} &  \textcolor{cyan}{-0.1419}   & \textcolor{cyan}{-0.1419}  &  \textcolor{brown}{0.3581}\\
   \textcolor{cyan}{-0.1419}   & \textcolor{red}{0.8581}  & \textcolor{cyan}{-0.1419}  & \textcolor{cyan}{-0.1419}  &  \textcolor{brown}{0.3581}\\
   \textcolor{cyan}{-0.1419} &  \textcolor{cyan}{-0.1419}  &  \textcolor{red}{0.8581} &  \textcolor{cyan}{-0.1419} &   \textcolor{brown}{0.3581}\\
   \textcolor{cyan}{-0.1419}  & \textcolor{cyan}{-0.1419}  & \textcolor{cyan}{-0.1419}  &  \textcolor{red}{0.8581}  &  \textcolor{brown}{0.3581}\\
   \textcolor{brown}{0.3581}  &  \textcolor{brown}{0.3581}  &  \textcolor{brown}{0.3581}  &  \textcolor{brown}{0.3581} &  \textcolor{blue}{-0.6419}\\
    \textcolor{gray}{0.1581}  &   \textcolor{gray}{0.1581}  &   \textcolor{gray}{0.1581}  &   \textcolor{gray}{0.1581}  &   \textcolor{gray}{0.1581}\\
  \textcolor{gray}{0.1581}  &   \textcolor{gray}{0.1581}  &   \textcolor{gray}{0.1581}  &   \textcolor{gray}{0.1581}  &   \textcolor{gray}{0.1581}\\
  \textcolor{gray}{0.1581}  &   \textcolor{gray}{0.1581}  &   \textcolor{gray}{0.1581}  &   \textcolor{gray}{0.1581}  &   \textcolor{gray}{0.1581}
    \end{bmatrix},\\
\mathbf{W}^{\epsilon_2}_{8\times 5}
&=
    \begin{bmatrix}
   \textcolor{red}{0.8618} &  \textcolor{cyan}{-0.1415} &  \textcolor{cyan}{-0.1413} &  \textcolor{cyan}{-0.1413} &    \textcolor{brown}{0.3524}\\
   \textcolor{cyan}{-0.1341} &   \textcolor{red}{0.8626}  & \textcolor{cyan}{-0.1374} &  \textcolor{cyan}{-0.1374}  &   \textcolor{brown}{0.3563}\\
   \textcolor{cyan}{-0.1342} &  \textcolor{cyan}{-0.1373}  &  \textcolor{red}{0.8626}  & \textcolor{cyan}{-0.1374}  &   \textcolor{brown}{0.3563}\\
   \textcolor{cyan}{-0.1342} &  \textcolor{cyan}{-0.1373}  & \textcolor{cyan}{-0.1373}  &  \textcolor{red}{0.8626} &    \textcolor{brown}{0.3563}\\
    \textcolor{brown}{0.3559} &    \textcolor{brown}{0.3528} &    \textcolor{brown}{0.3528}  &   \textcolor{brown}{0.3528} &  \textcolor{blue}{-0.6533}\\
    \textcolor{gray}{0.1598}  &  \textcolor{gray}{0.1567}  &  \textcolor{gray}{0.1567}  &  \textcolor{gray}{0.1567}  &  \textcolor{gray}{0.1506}\\
    \textcolor{gray}{0.1598}  &  \textcolor{gray}{0.1567}  &  \textcolor{gray}{0.1567}  &  \textcolor{gray}{0.1567}  &  \textcolor{gray}{0.1506}\\
    \textcolor{gray}{0.1598}  &  \textcolor{gray}{0.1567}  &  \textcolor{gray}{0.1567}  &  \textcolor{gray}{0.1567}  &  \textcolor{gray}{0.1506}
    \end{bmatrix}.
\end{align*}
\end{example}

\smallskip

\begin{figure}[t!]
\centering 
\includegraphics[width=0.47\textwidth ]{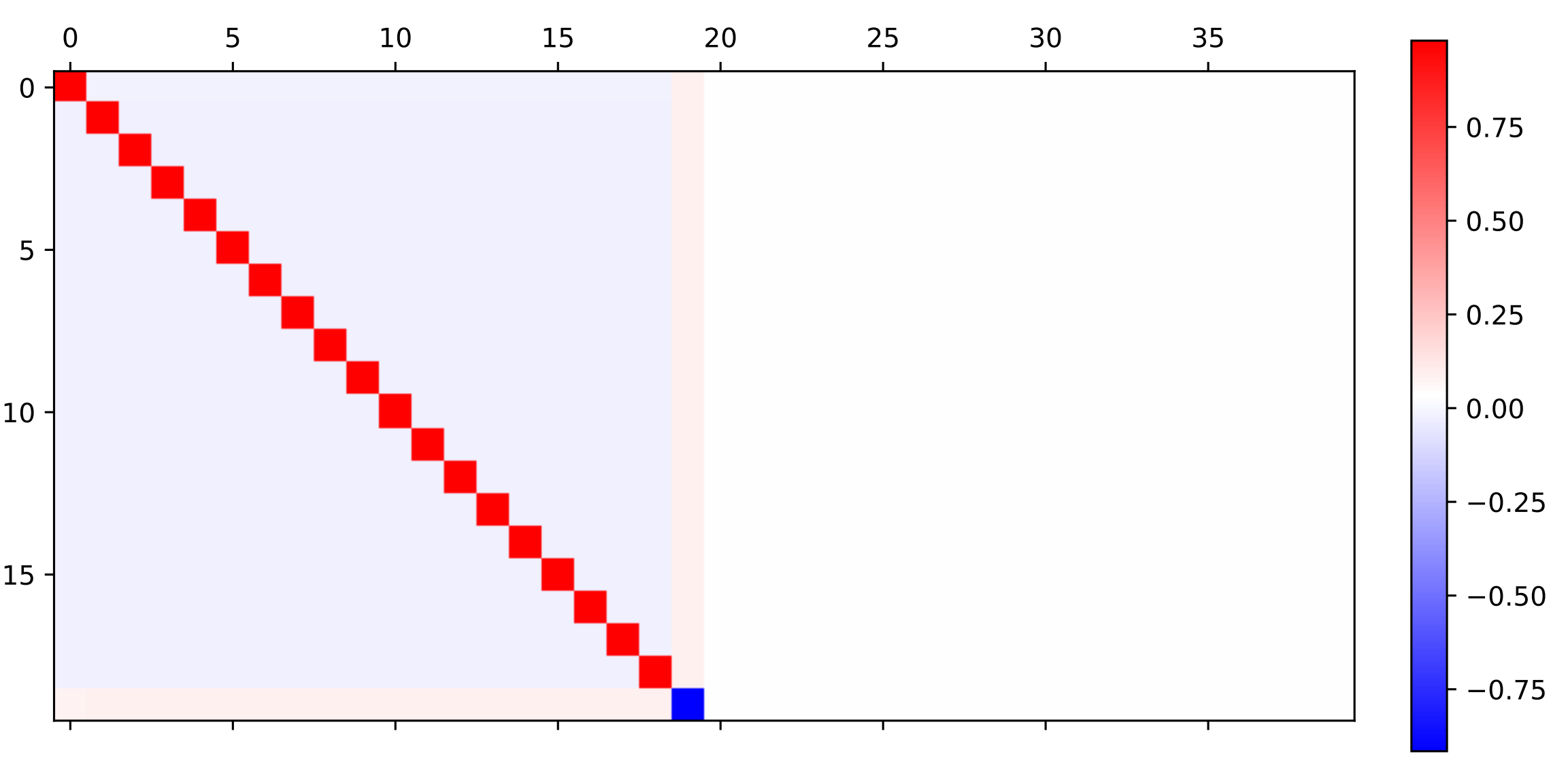}
\caption{A proposed initial weight matrix $\mathbf{W}^{\epsilon}_{20\times 40}$ is shown via heatmap ($\epsilon=0.01$). There exists a certain pattern of values for entries of $\mathbf{W}^{\epsilon}_{20\times 40}$. }
\label{fig:heatmap1}
\end{figure}

\section{Properties of the proposed initial weight matrix}\label{sec:property}
This section presents several key properties of the proposed initial weight matrix, accompanied by rigorous proofs. Initially, Proposition~\ref{prop:orthogonality} establishes the orthogonality of the proposed initial weight matrix $\mathbf{W}^{\epsilon}$. 
Furthermore, Theorem~\ref{thm1} introduces an algorithm designed to reduce the computational complexity of $\mathbf{W}^{\epsilon}$. Proposition~\ref{prop:csum} demonstrates that the column sums and row sums of $\mathbf{W}^{\epsilon}$ are nearly identical. Building on this, Corollary~\ref{cor1} shows that $\mathbf{W}^{\epsilon}\mathbf{x}$ has more positive entries for any vector $\mathbf{x}$ with positive entries.  

\begin{proposition}\label{prop:orthogonality}
Let $\mathbf{q}_1,\ldots,\mathbf{q}_m$ be the column vectors of $\mathbf{Q}_{m\times m}^{\epsilon}$ and $\hat{\mathbf{q}}_1,\ldots,\hat{\mathbf{q}}_n$ be the column vectors of $\mathbf{Q}_{n\times n}^{\epsilon}$. Then it holds that
\begin{itemize}
\item[$($i$)$] if $m=n$,
$$({\mathbf{W}^{\epsilon}_{m\times n }})^T \mathbf{W}^{\epsilon}_{m\times n } = \mathbf{W}^{\epsilon}_{m\times n }({\mathbf{W}^{\epsilon}_{m\times n }})^T 
=   \mathbf{I},$$   

\item[$($ii$)$] if $m>n$,  
\begin{align*}
    ({\mathbf{W}^{\epsilon}_{m\times n }})^T \mathbf{W}^{\epsilon}_{m\times n } 
    &= \mathbf{I}_{n\times n},\\
    \mathbf{W}^{\epsilon}_{m\times n } ({\mathbf{W}^{\epsilon}_{m\times n }})^T
    &= \mathbf{q}_1 \mathbf{q}_1^T + \mathbf{q}_2 \mathbf{q}_2^T + \cdots + \mathbf{q}_{n} \mathbf{q}_n^T,
\end{align*}

\item[$($iii$)$] if $m<n$,
\begin{align*}
    \mathbf{W}^{\epsilon}_{m\times n } ({\mathbf{W}^{\epsilon}_{m\times n }})^T 
    &= \mathbf{I}_{m\times m}\\
   ({\mathbf{W}^{\epsilon}_{m\times n }})^T \mathbf{W}^{\epsilon}_{m\times n } 
   &= \hat{\mathbf{q}}_1 \hat{\mathbf{q}}_1^T + \hat{\mathbf{q}}_2 \hat{\mathbf{q}}_2^T + \cdots + \hat{\mathbf{q}}_{m} \hat{\mathbf{q}}_m^T.
\end{align*}

\item[$($iv$)$] $(\mathbf{W}^{\epsilon}_{m\times n})^T=\mathbf{W}^{\epsilon}_{n\times m}$ for all $m,n$.
\end{itemize}
\label{prop:orthogonal}
\end{proposition}
It is easy to verify it. The proof is left to the reader.

\bigskip
Now, we introduce an algorithm that can reduce the computational complexity of calculating $\mathbf{W}^{\epsilon}$. Recall that for given vectors $\mathbf{u}$ and $\mathbf{v}$, the vector projection of $\mathbf{v}$ onto $\mathbf{u}$ is defined as 
\begin{equation*}
\text{proj}_{\mathbf{u}}{\mathbf{v}} := \frac{\left \langle \mathbf{u}, \mathbf{v} \right \rangle}{\left \langle \mathbf{u}, \mathbf{u} \right \rangle} \mathbf{u}.
\end{equation*}

QR decomposition is performed as follows.
For a $n\times n$ matrix $\mathbf{A}=[\mathbf{a}_1 \cdots \mathbf{a}_n]$, the QR decomposition is defined as 
\begin{equation*}
    \mathbf{A}=\mathbf{Q}\mathbf{R} \quad \text{($\mathbf{Q}$: orthogonal matrix,  $\mathbf{R}$: upper  triangular matrix)},
\end{equation*}
where
\begin{align*}
\mathbf{Q}&=
\underbrace{
\begin{bmatrix}
      |       &    |      &          &    |   \\ 
      \mathbf{q}_{1}   &   \mathbf{q}_{2}   &  \cdots  & \mathbf{q}_{n}   \\
      |       &    |      &          &    |  
\end{bmatrix} }_{\text{orthogonal matrix}},\\
\mathbf{R}&=
\underbrace{
\begin{bmatrix}
\mathbf{q}_{1}^{T} \cdot \mathbf{a}_{1}  &    \mathbf{q}_{1}^T \cdot \mathbf{a}_{2}   & \cdots  &   \mathbf{q}_{1}^T \cdot \mathbf{a}_{n}   \\ 
0                      &    \mathbf{q}_{2}^T \cdot \mathbf{a}_{2}   & \cdots  &   \mathbf{q}_{2}^T \cdot \mathbf{a}_{n}    \\
\vdots                 &      \vdots              & \ddots  &   \vdots                \\
0                      &    0                     & \cdots  &   \mathbf{q}_{n}^T \cdot \mathbf{a}_{n} 
\end{bmatrix} }_{\text{upper triangular matrix}}.
\end{align*}
Here, the matrices $\mathbf{Q}$ and $\mathbf{R}$ are generated by the Gram-Schmidt process for the full column rank matrix $\mathbf{A}=[\mathbf{a}_1 \cdots \mathbf{a}_n]$.

\begin{align*}
& \mathbf{u}_{1} = \mathbf{a}_{1}, \quad  &\mathbf{q}_{1}= \dfrac{\mathbf{u}_{1}}{\lVert \mathbf{u}_{1} \rVert},\\
& \mathbf{u}_{2} = \mathbf{a}_{2} - \text{proj}_{\mathbf{u}_1}{\mathbf{a}_2}, 
\quad  &\mathbf{q}_{2}= \dfrac{\mathbf{u}_{2}}{\lVert \mathbf{u}_{2} \rVert},\\
& \mathbf{u}_{3} = \mathbf{a}_{3} - \text{proj}_{\mathbf{u}_1}{\mathbf{a}_3} - \text{proj}_{\mathbf{u}_2}{\mathbf{a}_3}, 
\quad  &\mathbf{q}_{3}= \dfrac{\mathbf{u}_{3}}{\lVert \mathbf{u}_{3} \rVert},\\
& \qquad \qquad  \qquad \qquad \vdots\\
& \mathbf{u}_{n} = \mathbf{a}_{n} - \sum_{j=1}^{n-1}\text{proj}_{\mathbf{u}_j}{\mathbf{a}_n}, \quad  &\mathbf{q}_{n}= \dfrac{\mathbf{u}_{n}}{\lVert \mathbf{u}_{n} \rVert}.
\end{align*}
\medskip

From the Gram-Schmidt
process, we have the following iteration to construct $\mathbf{Q}_{m\times m}^{\epsilon}$. 
\begin{theorem}\label{thm1}
Let $\{\mathbf{u}_j\}_{1\leq j\leq m}$ be defined by
\begin{align*}
\mathbf{u}_1
&= \mathbf{1}+\epsilon \mathbf{e}_{1} \in \mathbb{R}^m, \\
    \mathbf{u}_j
    &= \left(  1- \frac{\left \langle \mathbf{u}_{j-1}, \mathbf{1}+\epsilon \mathbf{e}_{j} \right \rangle}{\left \langle \mathbf{u}_{j-1}, \mathbf{u}_{j-1} \right \rangle} \right) \mathbf{u}_{j-1}  +\epsilon (\mathbf{e}_{j}-\mathbf{e}_{j-1}) \in \mathbb{R}^m 
\end{align*}
for each $j=2,\ldots,m$.
Then $j$-th column vector of $\mathbf{Q}^{\epsilon}_{m\times m}$ is expressed as $\dfrac{1}{\|\mathbf{u}_j\|}\mathbf{u}_j$.
\end{theorem}
\begin{proof}
Let $\mathbf{a}_j$ be the $j$-th column vector of $ \mathbf{J}^{\epsilon}$, i.e.,
\begin{equation*}
    \mathbf{J}^{\epsilon}=
\begin{bmatrix}
    \mathbf{a}_1 & \ldots &  \mathbf{a}_n
\end{bmatrix}
=
\left[
\begin{bmatrix}
1+\epsilon \\
1\\
\vdots\\
1
\end{bmatrix},
\begin{bmatrix}
1\\
1+\epsilon \\
\vdots\\
1
\end{bmatrix},
\cdots,
\begin{bmatrix}
1\\
1\\
\vdots\\
1+\epsilon
\end{bmatrix}
\right].
\end{equation*}
Now we apply the Gram-Schmidt process to matrix $\mathbf{J}^{\epsilon}$.
Then the 1st orthogonal vector is given as
\begin{align*}
\mathbf{u}_1 = \mathbf{a}_1 = 
\begin{bmatrix}
1+\epsilon \\
1\\
\vdots\\
1
\end{bmatrix}.
\end{align*}
Next, the 2nd orthogonal vector is constructed as
\begin{align*}
\mathbf{u}_{2} 
&= \mathbf{a}_{2} - \text{proj}_{\mathbf{u}_1}{\mathbf{a}_2}\\ 
&= \mathbf{a}_{2} - \frac{\left \langle \mathbf{u}_1, \mathbf{a}_2 \right \rangle}{\left \langle \mathbf{u}_1, \mathbf{u}_1 \right \rangle}\mathbf{u}_1 \\
&= \mathbf{a}_1 + \epsilon (\mathbf{e}_2-\mathbf{e}_1)
- \frac{\left \langle \mathbf{u}_1, \mathbf{a}_2 \right \rangle}{\left \langle \mathbf{u}_1, \mathbf{u}_1 \right \rangle}\mathbf{u}_1\\
&=\left( 1-\frac{\left \langle \mathbf{u}_1, \mathbf{a}_2 \right \rangle}{\left \langle \mathbf{u}_1, \mathbf{u}_1 \right \rangle} \right) \mathbf{u}_1
+\epsilon (\mathbf{e}_2-\mathbf{e}_1).
\end{align*}

\begin{align*}
\mathbf{u}_{k} 
&= \mathbf{a}_{k} - \sum_{j=1}^{k-1}\text{proj}_{\mathbf{u}_j}{\mathbf{a}_n} \\ 
&= \mathbf{a}_{k} - \sum_{j=1}^{k-1} \frac{\left \langle \mathbf{u}_j, \mathbf{a}_k \right \rangle}{\left \langle \mathbf{u}_j, \mathbf{u}_j \right \rangle}\mathbf{u}_j\\
&=\mathbf{a}_{k-1}+\epsilon (\mathbf{e}_{k}-\mathbf{e}_{k-1}) 
-\sum_{j=1}^{k-1} \frac{\left \langle \mathbf{u}_j, \mathbf{a}_k \right \rangle}{\left \langle \mathbf{u}_j, \mathbf{u}_j \right \rangle}\mathbf{u}_j\\
&=\mathbf{a}_{k-1}
-\sum_{j=1}^{k-2} \frac{\left \langle \mathbf{u}_j, \mathbf{a}_k \right \rangle}{\left \langle \mathbf{u}_j, \mathbf{u}_j \right \rangle}\mathbf{u}_j
- \frac{\left \langle \mathbf{u}_{k-1}, \mathbf{a}_{k} \right \rangle}{\left \langle \mathbf{u}_{k-1}, \mathbf{u}_{k-1} \right \rangle}\mathbf{u}_{k-1} +\epsilon (\mathbf{e}_{k}-\mathbf{e}_{k-1})\\ 
&=\mathbf{a}_{k-1}
-\sum_{j=1}^{k-2} \frac{\left \langle \mathbf{u}_j, \mathbf{a}_{k-1} \right \rangle}{\left \langle \mathbf{u}_j, \mathbf{u}_j \right \rangle}\mathbf{u}_j
- \frac{\left \langle \mathbf{u}_{k-1}, \mathbf{a}_{k} \right \rangle}{\left \langle \mathbf{u}_{k-1}, \mathbf{u}_{k-1} \right \rangle}\mathbf{u}_{k-1} +\epsilon (\mathbf{e}_{k}-\mathbf{e}_{k-1})\\ 
&=\mathbf{u}_{k-1}
- \frac{\left \langle \mathbf{u}_{k-1}, \mathbf{a}_{k} \right \rangle}{\left \langle \mathbf{u}_{k-1}, \mathbf{u}_{k-1} \right \rangle}\mathbf{u}_{k-1} +\epsilon (\mathbf{e}_{k}-\mathbf{e}_{k-1}). 
\end{align*}
The last second equality holds from the fact that all $i$-th ($i\geq j$) entries of $\mathbf{u}_{j-1}$ are identical.   
\end{proof}

The iteration in Theorem~\ref{thm1} can reduce the computational complexity of \(\mathbf{W}^{\epsilon}\). Through this theorem, the computational complexity of QR decomposition is reduced from \(\mathcal{O}(n^3)\) to \(\mathcal{O}(n^2)\). Furthermore, since the proposed initialization method is fully deterministic, it is sufficient to compute the matrix only once for a given epsilon $\epsilon$ and dimension of the matrix $m,n$, allowing it reused. 
Before demonstrating that the proposed initial weights have nearly equal column sums and row sums, we first establish the properties of $\mathbf{Q}^{\epsilon}$. 

\bigskip

Next, we prove the bound on the column sum of $\mathbf{Q}^{\epsilon}$ in Lemma~\ref{lemma:sum_Q1}, and from Lemma~\ref{lemma:sum_Q1}, we demonstrate Proposition~\ref{prop:csum}: The entry sum of each column (resp. row) vector of $\mathbf{W}^{\epsilon}$ is almost same.
\begin{lemma}
Let $\mathbf{a}_j$ be the $j$th column vector of $\mathbf{J}^{\epsilon}$ for $j=1,\ldots, m$.
Then
\begin{equation*}
\left \langle \frac{\mathbf{a}_j}{\| \mathbf{a}_j\|}, \frac{\mathbf{1}}{\|\mathbf{1} \|} \right\rangle = \frac{m+\epsilon}{\sqrt{m}\sqrt{m+2\epsilon+\epsilon^2}} = 1-\frac{m-1}{2m^2}\epsilon^2 +\mathcal{O}(\epsilon^3).
\end{equation*}
Furthermore, if $\epsilon\to 0$, then 
$$\left \langle \frac{\mathbf{a}_j}{\| \mathbf{a}_j\|}, \frac{\mathbf{1}}{\|\mathbf{1} \|} \right\rangle \to 1,$$
provided that $m$ is fixed.
\end{lemma}
\smallskip

\begin{lemma}\label{lemma:sum_Q1}
Let $\mathbf{q}_1,\ldots,\mathbf{q}_m$ be the column vectors of $\mathbf{Q}_{m\times m}^{\epsilon}$. Then it holds that
\begin{align*}
\left\langle \mathbf{q}_1, \mathbf{1} \right\rangle
        &= \frac{m+\epsilon}{\sqrt{\epsilon^2+2\epsilon+m}},\\
       \left| \left\langle \mathbf{q}_j, \mathbf{1} \right\rangle  \right|
       &\leq \epsilon \quad \text{for all }j=2,\ldots, m.
\end{align*}
\end{lemma}
\begin{proof}
By QR decomposition we have 
\begin{equation}
\mathbf{J}^{\epsilon}=\mathbf{J} + \epsilon \mathbf{I} = \mathbf{Q}^{\epsilon} \mathbf{R}^{\epsilon},
\end{equation}
where $\mathbf{Q}^{\epsilon}$ is the orthogonal matrix and $\mathbf{R}^{\epsilon}$ is the upper triangular matrix.
By multiplying $(\mathbf{Q}^{\epsilon})^T$ and $\mathbf{1}$ on both sides, it follows that  
\begin{equation*}
    (\mathbf{Q}^{\epsilon})^T(\mathbf{J} + \epsilon \mathbf{I}) \mathbf{1}
    = \mathbf{R}^{\epsilon} \mathbf{1}.
\end{equation*}
Let $q_{ij}$ be the entry in $i$-th row and $j$-th column of $\mathbf{Q}_{m\times m}^{\epsilon}$ and 
let $\mathbf{q}_1,\ldots,\mathbf{q}_m$ are the column vectors of $\mathbf{Q}_{m\times m}^{\epsilon}$.
Thus we have
\begin{equation}
    (m+\epsilon)(\mathbf{Q}^{\epsilon})^T  \mathbf{1}  = 
    \begin{bmatrix}
        \langle \mathbf{q}_1, \mathbf{v}_1 \rangle \\
         \langle \mathbf{q}_2, \mathbf{v}_2 \rangle \\
         \vdots\\
          \langle \mathbf{q}_m, \mathbf{v}_m \rangle
    \end{bmatrix},
\end{equation}
where for each $k=1,\ldots,m$  
\begin{equation*}
    \mathbf{v}_k=
    \begin{bmatrix}
    m-k+1\\
    m-k+1\\
    \vdots\\
    m-k+1\\
    m-k+1+\epsilon\\
    \vdots\\
    m-k+1+\epsilon
    \end{bmatrix}.
\end{equation*}
Let $S_j$ be the sum of all entries of $\mathbf{q}_j$. 
Then we have that for all $j=2,3,\ldots,m$
\begin{align*}
S_j &= \frac{1}{m+\epsilon}  \left\langle \mathbf{q}_j, \mathbf{v}_j \right\rangle \\
&= \frac{1}{m+\epsilon} 
((m-j+1)q_{1j} + \cdots +(m-j+1 +\epsilon) q_{jj}  \\
& \quad + \cdots +(m-j+1 +\epsilon) q_{mj} )\\
& = \frac{m-j+1+\epsilon}{m+\epsilon} S_j -\frac{\epsilon}{m +\epsilon}
\left(q_{1j} + \cdots+ q_{j-1, j}
\right),
\end{align*}
implying that
\begin{equation*}
\left(1- \frac{m-j+1+\epsilon}{m+\epsilon} 
\right) S_j 
= -\frac{\epsilon}{m +\epsilon}
\left(q_{1j} + \cdots +q_{j-1, j}\right).
\end{equation*}
Thus it follows that for all $j=2,3,\ldots,m$
\begin{align}
\left | S_j \right|  
&= \frac{\epsilon}{ j-1}\left |q_{1j} + \cdots +q_{j-1, j}\right| \nonumber\\
&\leq \frac{\epsilon}{ j-1} \left( | q_{1j} | + \cdots +| q_{j-1, j}| \right)   \nonumber \\
& \leq \frac{\epsilon}{ \sqrt{j-1}}\leq \epsilon. \label{eq:sum_ineq1}
\end{align}
The first inequality and the second inequality hold from 
the triangle inequality and the Cauchy-Schwarz inequality, respectively.
\end{proof}

\begin{proposition}\label{prop:csum}
The entry sum of each column (resp. row) vector of  $\mathbf{W}^{\epsilon}_{m\times n}$ 
is almost the same.
\end{proposition}

\begin{proof}
Let $\mathbf{c}_j$ be $j$-th column vector of $\mathbf{W}^{\epsilon}=\mathbf{W}^{\epsilon}_{m\times n}$ for each $j=1,\ldots,n$.
Then by \eqref{eq:outerprod1} the sum of all entries of $j$-th column vector can be expressed as
\begin{equation*}
    \mathbf{c}_j^T\mathbf{1} = \mathbf{e}_j^T (\mathbf{W}^{\epsilon})^T \mathbf{1} 
    = \mathbf{e}_j^T \left( \hat{\mathbf{q}}_1 \mathbf{q}_1^T + \cdots + \hat{\mathbf{q}_{s}} \mathbf{q}_s^T    \right)\mathbf{1},  
\end{equation*}
where $s=\min\{m,n\}$. So, it follows that for each $j=1,\ldots,n$
\begin{align*}
    \left| \mathbf{c}_j^T\mathbf{1} - \mathbf{e}_j^T \hat{\mathbf{q}}_1 \mathbf{q}_1^T \mathbf{1} \right|
    &=\left| \mathbf{e}_j^T (\mathbf{W}^{\epsilon})^T \mathbf{1} - \mathbf{e}_j^T \hat{\mathbf{q}}_1 \mathbf{q}_1^T \mathbf{1} \right|\\
    &=\left| \mathbf{e}_j^T \left( \hat{\mathbf{q}}_2 \mathbf{q}_2^T + \cdots + \hat{\mathbf{q}_{s}} \mathbf{q}_s^T   \right)\mathbf{1}   \right|\\
    & \leq \left| \mathbf{e}_j^T \hat{\mathbf{q}}_2  \right|
    \left| \mathbf{q}_2^T \mathbf{1}   \right|
    + \cdots + 
  \left| \mathbf{e}_j^T \hat{\mathbf{q}}_s  \right|
    \left| \mathbf{q}_s^T \mathbf{1}   \right|\\
    &= \left| \hat{q}_{j2} \right|
    \left| \langle \mathbf{q}_2, \mathbf{1} \rangle   \right|
    + \cdots + 
\left| \hat{q}_{js} \right|
    \left| \langle \mathbf{q}_s, \mathbf{1} \rangle   \right|\\
    &\leq \left| \hat{q}_{j2} \right|
    \frac{\epsilon}{\sqrt{1}}
    + \cdots + 
 \left| \hat{q}_{js} \right|
    \frac{\epsilon}{\sqrt{s-1}}\\
    &\leq \epsilon\sqrt{\frac{1}{1}+\cdots+\frac{1}{s-1}}
    \sqrt{\left| \hat{q}_{j2} \right|^2+\cdots +\left| \hat{q}_{js} \right|^2 }\\
   &= \epsilon \sqrt{H_{s-1}} \approx  \epsilon \sqrt{\log(s-1)} ,
\end{align*}

where $\hat{q}_{ij}$ is the $(i,j)$-entry of $\mathbf{Q}_{n\times n}^{\epsilon}$ and $H_{k}$ is the $k$-th harmonic number. 
The last equality holds from the fact that 
$\hat{\mathbf{q}}_1,\ldots,\hat{\mathbf{q}}_n$ are the column vectors of the orthogonal matrix $\mathbf{Q}_{n\times n}^{\epsilon}$, and the last inequality holds from the Cauchy-Schwarz inequality and by \eqref{eq:sum_ineq1} the second last inequality holds.

In a similar way, one can check that for each $i=1,\ldots,m$ it holds that

\begin{equation*}
      \left| \mathbf{r}_i^T\mathbf{1} - \mathbf{e}_i^T \mathbf{q}_1 \hat{\mathbf{q}}_1^T  \mathbf{1} \right| \leq \epsilon \sqrt{H_{s-1}},
\end{equation*}
where $\mathbf{r}_i$ be the $i$-th row vector of $\mathbf{W}^{\epsilon}$ for each $i=1,\ldots,m$.
\end{proof}

Now, we have established that the entry sum of each column vector of $\mathbf{W}^{\epsilon}_{m\times n}$ is almost the same. Note that the entry sum of a given column vector can be expressed as the inner product of the column vector and $\mathbf{1}$. 
It means that the column vectors of  $\mathbf{W}^{\epsilon}_{m\times n}$ are located in the vicinity of a plane that forms a specific angle with the vector $\mathbf{1}$.

\bigskip

The challenges associated with training deep neural networks using ReLU arise from the limitation that negative signals cannot propagate through the network. Our proposed weight matrix can make positive signals propagate through the network. Firstly, let us give experimental results for positive signal propagation with the proposed weight matrix and a Gaussian random matrix, orthogonal matrix. We set $\mathbf{W}\in\mathbb{R}^{200\times 100}$ as our proposed initial weight matrix with $\epsilon = 0.1$, a Gaussian random matrix with a mean of $0$ a standard deviation of $0.1$, and an orthogonal matrix. The positive signals $\mathbf{x}\in\mathbb{R}^{100}$ are generated from (i) normal distribution $\mathcal{N}(0.5, 0.25^2)$ and (ii) uniform distribution $\mathcal{U}_{[0,1]}$. For a random vector $\mathbf{x} \in \mathbb{R}^{100}$, the entry values of $\mathbf{Wx}$ are computed over 25 times for three different weight matrices. The computational results confirm that our proposed weight initialization method consistently yields a higher number of positive entries in $\mathbf{Wx}$ than other matrices, demonstrating its effectiveness in facilitating signal propagation in networks with ReLU activation (see Figure~\ref{fig:relu}).

\begin{figure}[t!]
\centering 
\begin{tabular}{cc}
\subfloat[$\mathbf{W}:$ Proposed, $\mathbf{x}:$ Gaussian]{{\includegraphics[width=0.235\textwidth ]{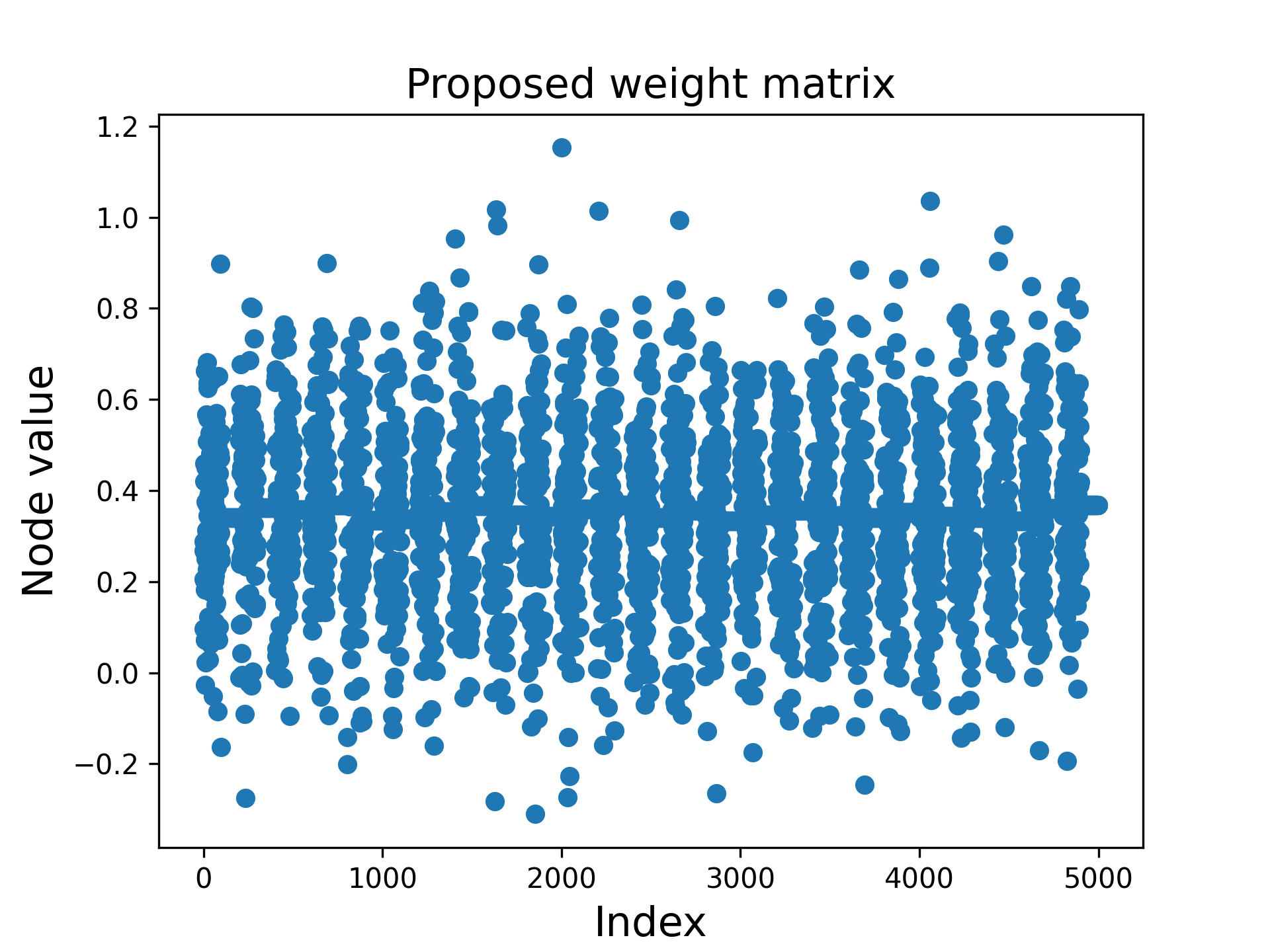} }}
\subfloat[$\mathbf{W}:$ Prosoed, $\mathbf{x}:$ Uniform]{{\includegraphics[width=0.235\textwidth ]{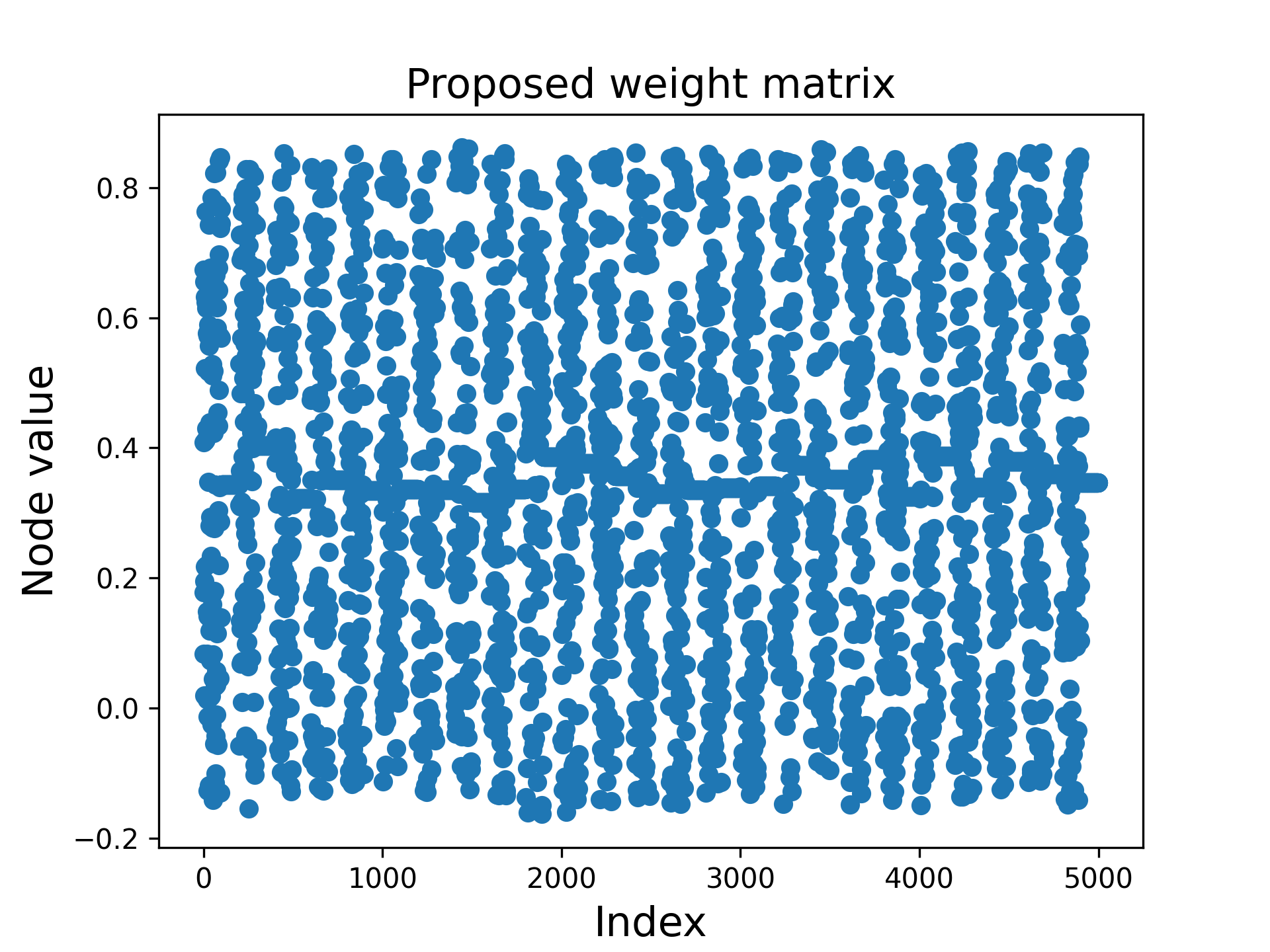} }}\\
\subfloat[$\mathbf{W}:$ Gaussian, $\mathbf{x}:$ Gaussian]{{\includegraphics[width=0.235\textwidth ]{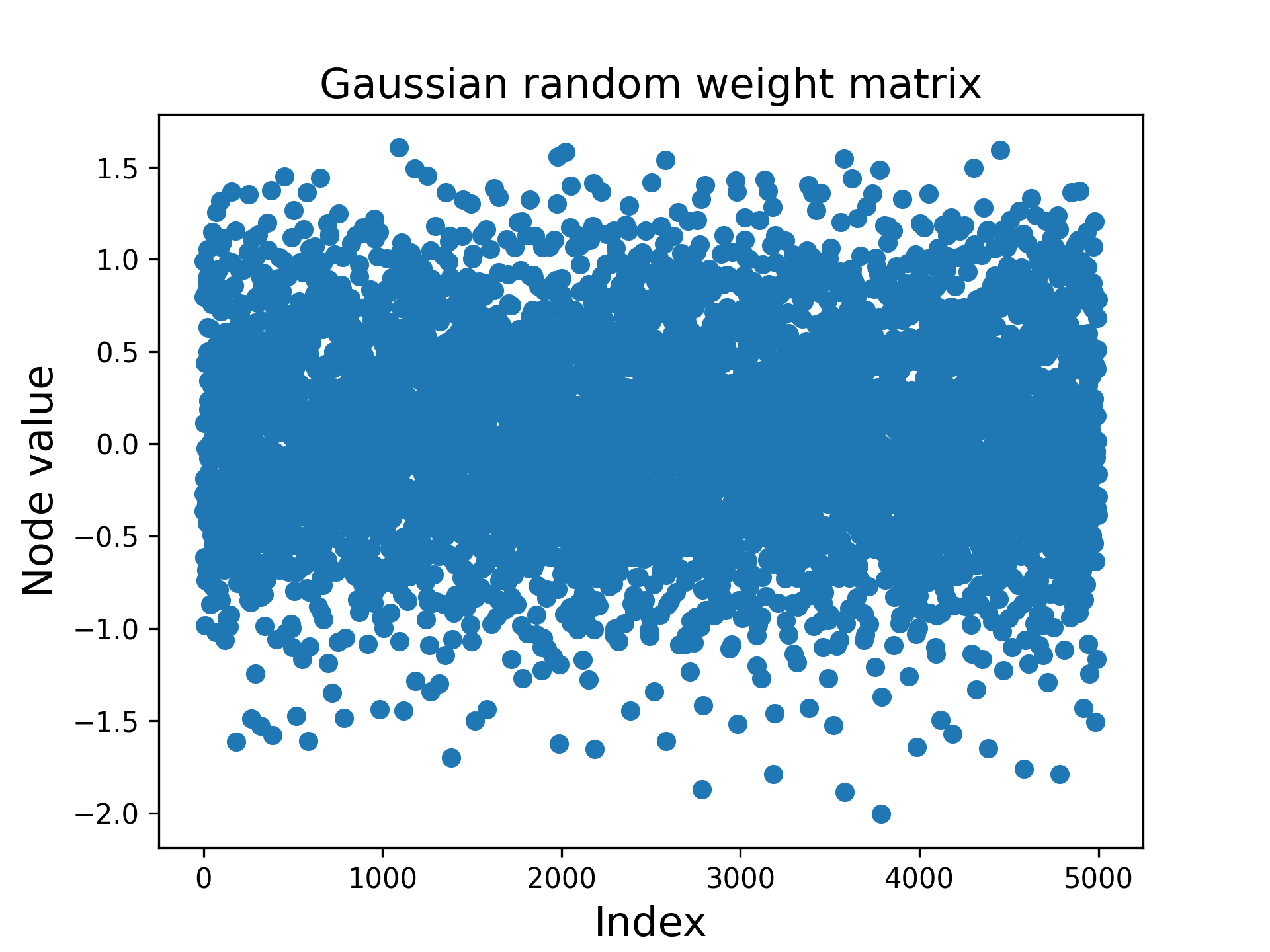} }}
\subfloat[$\mathbf{W}:$ Gaussian, $\mathbf{x}:$ Uniform]{{\includegraphics[width=0.235\textwidth ]{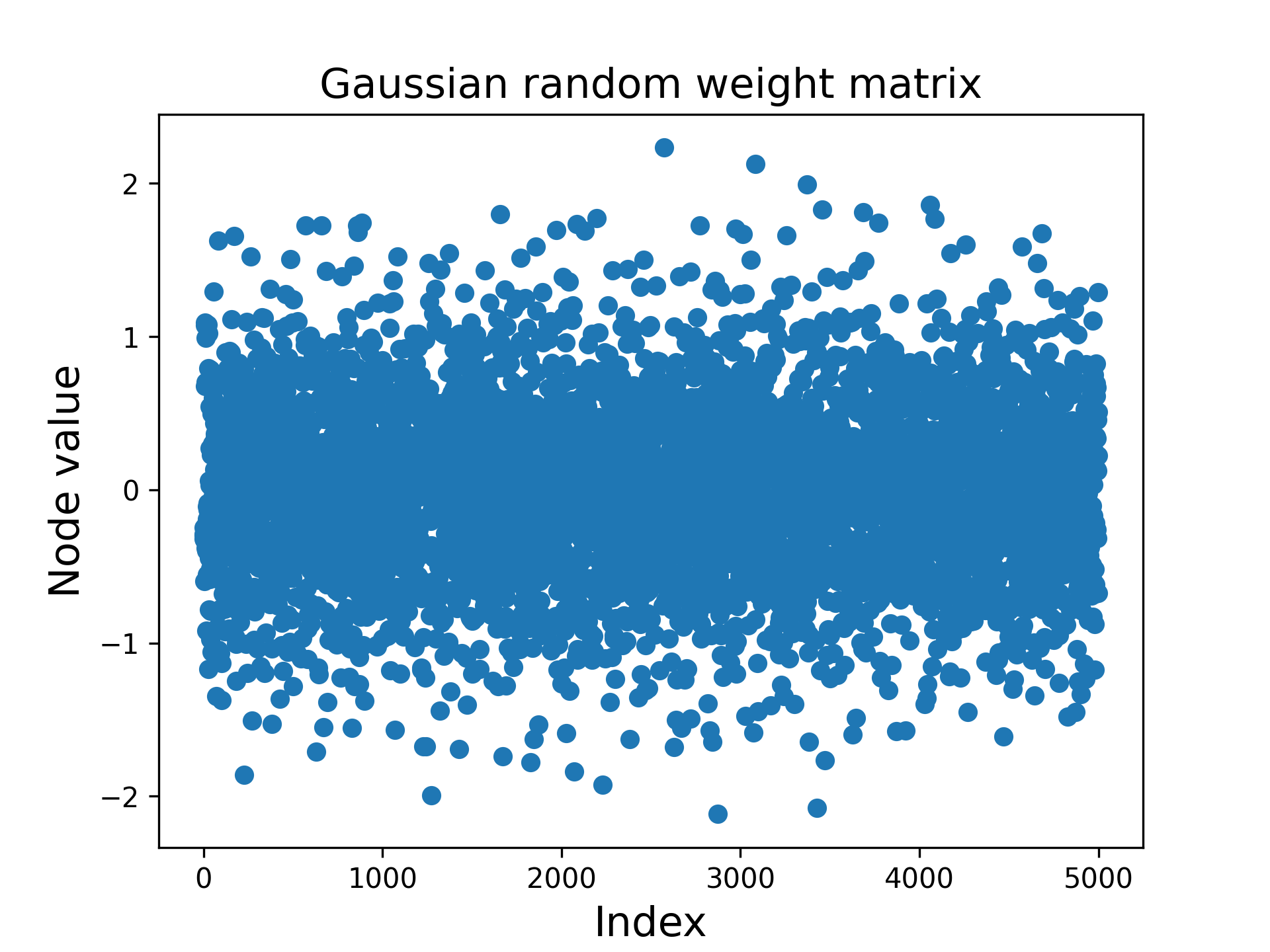} }}\\
\subfloat[$\mathbf{W}:$ Orthogonal, $\mathbf{x}:$ Gaussian]{{\includegraphics[width=0.235\textwidth ]{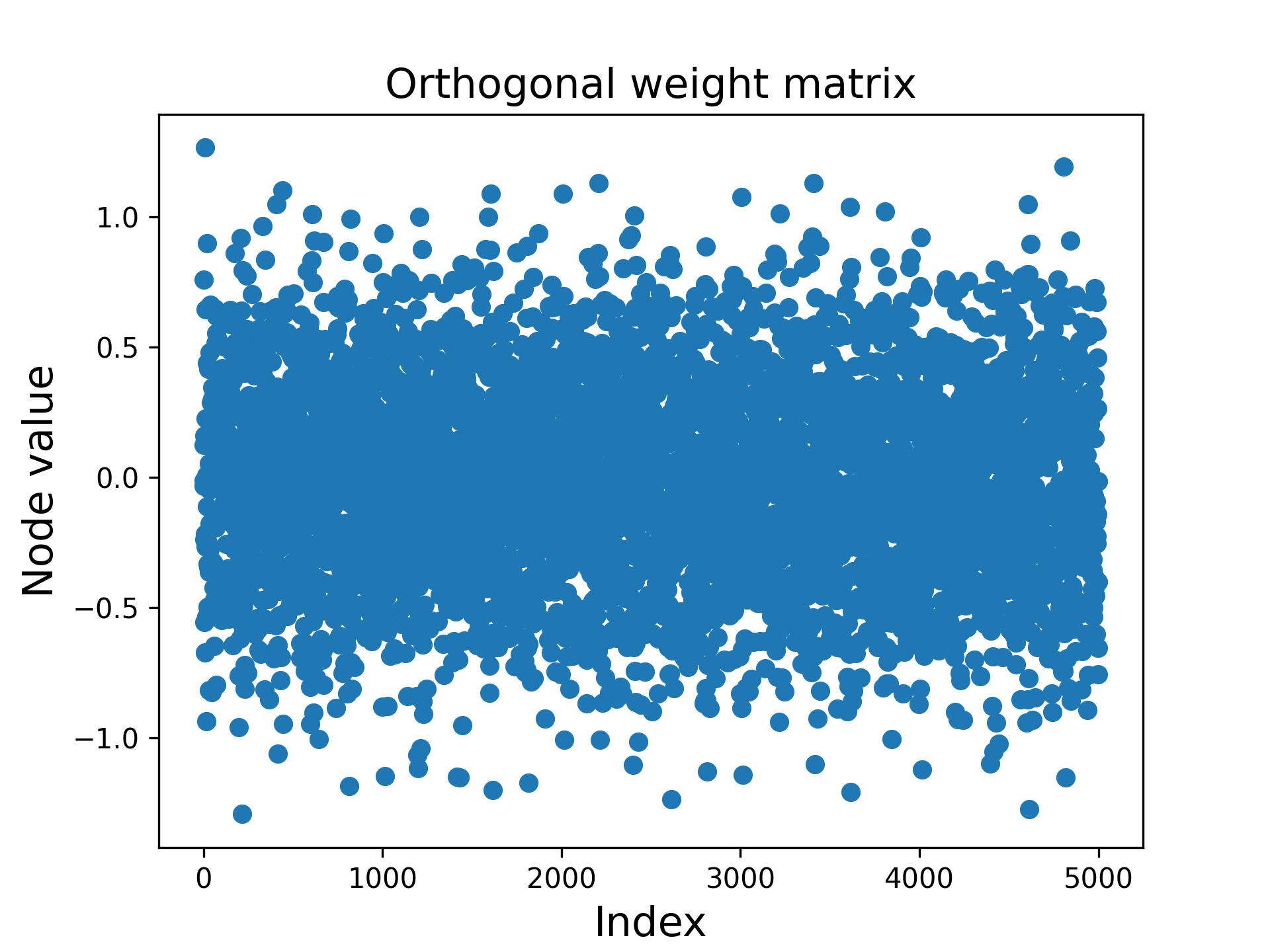} }}
\subfloat[$\mathbf{W}:$ Orthogonal, $\mathbf{x}:$ Uniform]{{\includegraphics[width=0.235\textwidth ]{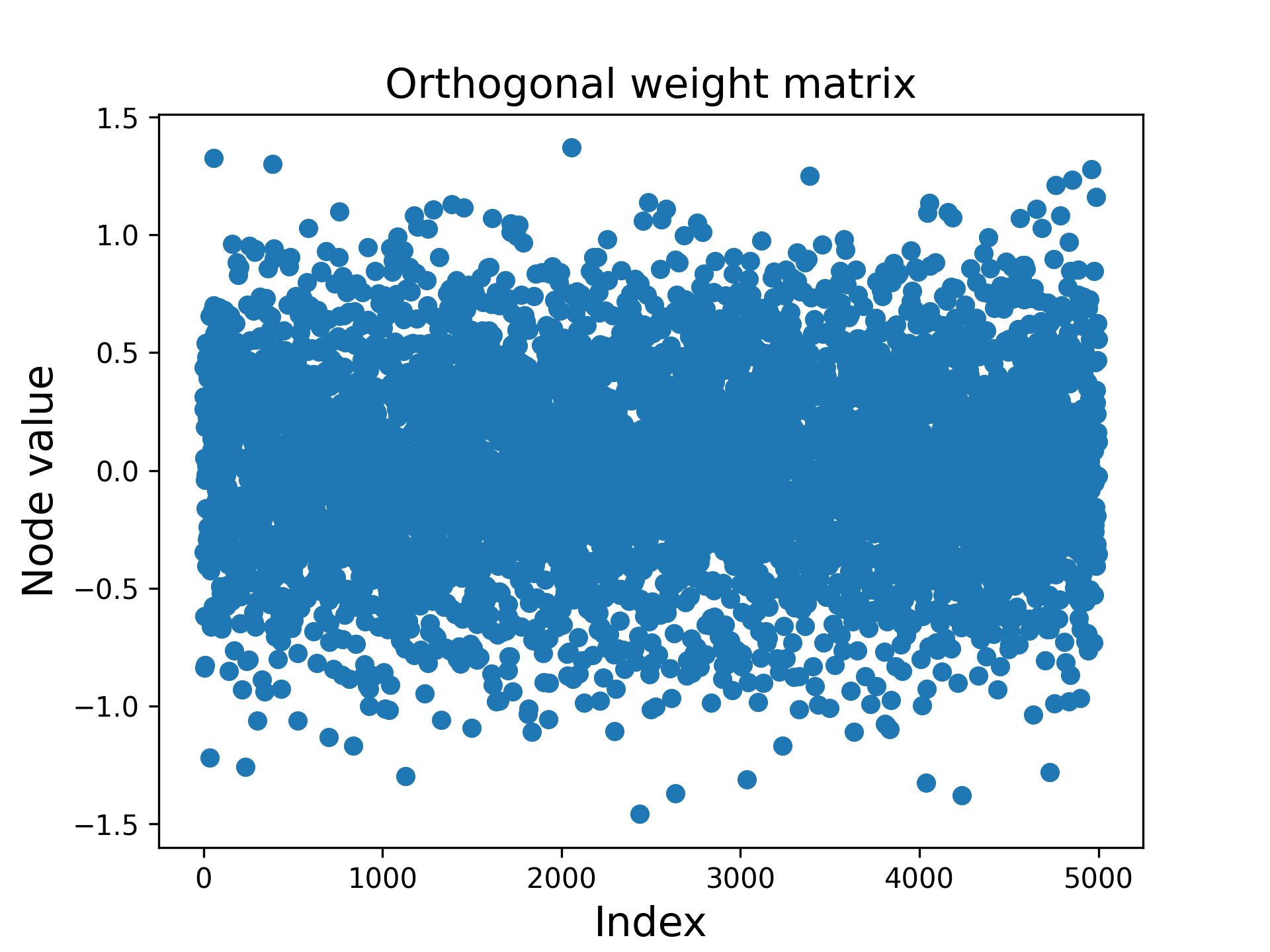} }}
\end{tabular}
\caption{
This shows its effectiveness of positive signal propagation for each weight matrices $\mathbf{W} \in \mathbb{R}^{200 \times 100}$. For 25 random vectors $\mathbf{x} \in \mathbb{R}^{100}$, the entry values of $\mathbf{Wx}$ are plotted.
Here, the $x$-axis represents the indices of all entries.}
\label{fig:relu}
\end{figure}

Now, we provide a theoretical analysis demonstrating that $\mathbf{W}^{\epsilon}\mathbf{x}$ contains a higher number of positive entries.

\begin{table*}[th]
\centering
\resizebox{15cm}{!}{% 
\begin{tabular}{l cc cc cc cc cc cc cc cc}
\toprule
Entire dataset  \\
\midrule
 & \multicolumn{2}{c}{Proposed} & \multicolumn{2}{c}{Orthogonal}
 & \multicolumn{2}{c}{Xavier}& \multicolumn{2}{c}{He}
 & \multicolumn{2}{c}{Zero}& \multicolumn{2}{c}{Identity}
 & \multicolumn{2}{c}{RAI}& \multicolumn{2}{c}{GSM}\\
\cmidrule(lr){2-17} 
Datset & 10&100& 10&100& 10&100&10&100&10&100&10&100&10&100&10&100   \\
\midrule
MNIST (0)      & \textbf{88.2}&88.5&88.1&88.7& 86.6&87.7&87.9&\textbf{89.1}&{88.1}&{88.8}&88&89.1&88&88.2&87.1&89.4  \\
FMNIST (0)     & \textbf{80.4}&\textbf{80.6}& 79.6&79& 79.5&78.1&79.5&80.1&{78.1}&{80.4}&78.4&80.3&79.1&80.4&75.8&80 \\
MNIST (512)     & \textbf{96.5}&\textbf{97.6}& 95.8&96.3& 95.8&96.5&96.4&96.6&95.1&96.5&96.2&96.5 &95.9&97.5&88&89.4 \\
FMNIST (512)      & 84.5&85.1& 84.4&85.4& 84.5&84.5&84.2&85.1&84.5&84.6&84.8&84.9&\textbf{84.9}&\textbf{85.2}&78.3&80.3 \\
MNIST   (16)    & \textbf{92.2}&\textbf{94}& 88&90& 83.5&86&77.2&85.5&60.1&84.2&38.7&40.5&91&92.1&29&77.1  \\
FMNIST  (16)    & \textbf{82.3}&\textbf{84.2}& 61.1&67.3& 56.4&69.2&53.3&60.7&60.1&83.1&78.1&81.2&60&78.4&34.9&37.3 \\
\midrule
$4$ samples per class \\
\midrule
 & \multicolumn{2}{c}{Proposed} & \multicolumn{2}{c}{Orthogonal}
 & \multicolumn{2}{c}{Xavier}& \multicolumn{2}{c}{He}
 & \multicolumn{2}{c}{Zero}& \multicolumn{2}{c}{Identity}
 & \multicolumn{2}{c}{RAI}& \multicolumn{2}{c}{GSM}\\
\cmidrule(lr){2-17} 
Datset & 10&100& 10&100& 10&100&10&100&10&100&10&100&10&100&10&100   \\
\midrule
MNIST  (0)     & \textbf{55.5}&\textbf{54.1}& 29.1&39.7& 26.1&40.5&23.1&38.7&51.1&50.8&54.2&53.5&27.6&43.8 &26.1&45.6 \\
FMNIST   (0)   &\textbf{54.2}&\textbf{57.1}& 42.7&51.4& 29.5&48.1&34.2&51.1&51&50.1&52.8&56.1&36.5&53.9 &33.7&51.4\\
MNIST (512)      & \textbf{56.5}&\textbf{51.0}& 49.7&50.1& 44.3&45.2&46.5&48.9&22&46.3&51.9&50.8&29.9&38.8 &23.3&37.2 \\
FMNIST (512)     & 46.7&55.6& 51&56& \textbf{54}&54.6&51&\textbf{56.8}&37.1&50.4&45.2&53.4 &48.7&56.2&45.1&53.6\\
MNIST  (16)     & \textbf{51.2}&\textbf{52.9}& 22.5&31.7& 18.7&26.3&20&25&9.1&10.3&9.6&10.3 &13.7&25.1 &11.9&18.8\\
FMNIST  (16)    & \textbf{43.3}&\textbf{56.3}& 23.4&24.7& 18.8&17.8&20&20.8&10.8&10.7&33.3&41.5&14.9&21&10.6&26.4 \\
\midrule
$2$ samples per class   \\
\midrule
 & \multicolumn{2}{c}{Proposed} & \multicolumn{2}{c}{Orthogonal}
 & \multicolumn{2}{c}{Xavier}& \multicolumn{2}{c}{He}
 & \multicolumn{2}{c}{Zero}& \multicolumn{2}{c}{Identity}
 & \multicolumn{2}{c}{RAI}& \multicolumn{2}{c}{GSM}\\
\cmidrule(lr){2-17} 
Datset & 10&100& 10&100& 10&100&10&100&10&100&10&100&10&100&10&100   \\
\midrule
MNIST  (0)     & \textbf{46.4}&\textbf{46.5}& 23.7&31.6& 19.6&30.1&20.7&28.7&42.6&43.3&44.1&43.6 &26.3&37.8&21.2&34.5 \\
FMNIST  (0)    & \textbf{49.1}&\textbf{50.3}& 38.3&43.3& 31.4&41.7&27.5&38.6&43&46.1&45.5&42.7 &36&40&38.2&44.7\\
MNIST (512)      & 39.7&37.1& 33.8&36.3& 32.7&33.1&39.3&39.1&27.2&33.4&\textbf{45.9}&\textbf{45.4} &38.5&42.4&37.7&41.1 \\
FMNIST (512)     & \textbf{46.8}&46.2& 45&\textbf{48.4}& 43.4&44.7&42.4&51.2&34.7&43.8&44,2&47.8&38.8&39.9&40.1&42.6 \\
MNIST  (16)     & \textbf{44.3}&\textbf{41.5}& 19.7&23.6& 16.6&21.6&19.3&22.2&10.1&11&9.6&9.5&11.2&22.8 &12.5&22.6 \\
FMNIST (16)     & \textbf{43.8}&\textbf{47.9}& 22.1&26.1& 18.6&20&19.4&22.7&9.9&10.5&29.1&39.6&24&26.6 &13.4&21.9\\
\midrule
$1$ samples per class  \\
\midrule
 & \multicolumn{2}{c}{Proposed} & \multicolumn{2}{c}{Orthogonal}
 & \multicolumn{2}{c}{Xavier}& \multicolumn{2}{c}{He}
 & \multicolumn{2}{c}{Zero}& \multicolumn{2}{c}{Identity}
 & \multicolumn{2}{c}{RAI}& \multicolumn{2}{c}{GSM}\\
\cmidrule(lr){2-17} 
Datset & 10&100& 10&100& 10&100&10&100&10&100&10&100&10&100&10&100   \\
\midrule
MNIST (0)    & \textbf{37.1}&\textbf{38.2}& 20.6&23.9 & 29.1&36.9&25.4&33.8&9.7&33&12.6&12.4&19.1&23.6&23.1&27.4  \\
FMNIST (0)  & \textbf{43.5}&39.4&  30.3&33.7& 27.4&30.8&24.6&35.8&9.7&25.8&40.7&\textbf{40.6}&18&33.7&34.5&40.6 \\
MNIST (512)   & 36.1&34.9& 28.2&27.7 & 31.2&32.3&27&27.4&22.2&29.8&\textbf{39.2}&\textbf{40.3}&32.5&29.3 &31.6&36.6 \\
FMNIST (512)  & \textbf{39.2}&37.4&  36.7&34.7& 38.5&\textbf{37.6}&36.1&35&31.7&37&0.3&3.4 &39&37.2&35.2&36\\
MNIST (16)   & \textbf{33.5}&\textbf{34.2}& 16.5&19.4 & 14.3&16.8&14.3&19.9&10.6&11.6&10&9.8&18.1&22.6&18.4&19.8  \\
FMNIST (16)  & \textbf{35}&\textbf{34.2}&  18.7&22.9& 16.1&16.8&19.8&22.8&10.3&10.5&7&7.2&13.7&19.7&15.9&21.9 \\

\bottomrule 
\end{tabular}}
\caption{This is a comparison of the validation accuracy for feedforward neural networks~(FFNNs) with various weight initialization methods. 
Here, $(\cdot)$ represents the number of nodes in a single hidden layer. The simulations are performed with datasets MNIST and FMNIST over $10$ epochs and $100$ epochs. Best results are marked in bold.}
\label{table:nohidden}
\end{table*}

\begin{theorem}\label{thm2}
Let $\mathbf{W}^{\epsilon}\in \mathbb{R}^{N_1\times N_x}$ with sufficiently small $\epsilon$ be a given. Then it holds that for all  $ \mathbf{x}\in \mathbb{R}^{N_x}$
\begin{equation*}
    \displaystyle\frac{1}{N_x}\left\langle \mathbf{x}, \mathbf{1}_{N_x} \right\rangle\simeq 
\sqrt{\frac{N_1}{N_x}}\frac{1}{N_1}\left\langle \mathbf{W}^{\epsilon}\mathbf{x}, \mathbf{1}_{N_1}\right\rangle.
\end{equation*}
\end{theorem}

\begin{proof}
    Since the proposed weight matrix $\mathbf{W}^{\epsilon}\in \mathbb{R}^{N_1\times N_x}$ holds orthogonality, it holds that
\begin{align*}
    \frac{1}{N_x} \left\langle \mathbf{x}, \mathbf{1}_{N_x} \right\rangle
    &=\frac{1}{N_x}\left\langle \mathbf{W}^{\epsilon}\mathbf{1}_{N_x}, \mathbf{W}^{\epsilon}\mathbf{x}  \right\rangle\\
    &= \frac{1}{N_x} \left\langle  (\mathbf{q}_1 \hat{\mathbf{q}}_1^T + \mathbf{q}_2 \hat{\mathbf{q}}_2^T + \cdots + \mathbf{q}_{s} \hat{\mathbf{q}}_s^T)\mathbf{1}_{N_x}, \mathbf{W}^{\epsilon}\mathbf{x}  \right\rangle \\ 
    &= \frac{1}{N_x} \left\langle\mathbf{q}_1 \hat{\mathbf{q}}_1^T\mathbf{1}_{N_x}, \mathbf{W}^{\epsilon}\mathbf{x} \right\rangle 
    + \frac{1}{N_x}\sum_{i=2}^{s} \left\langle \mathbf{q}_i \hat{\mathbf{q}}_i^T\mathbf{1}_{N_x}, \mathbf{W}^{\epsilon}\mathbf{x} \right\rangle\\
    &\simeq\frac{1}{N_x} \left\langle\mathbf{q}_1 \hat{\mathbf{q}}_1^T\mathbf{1}_{N_x}, \mathbf{W}^{\epsilon}\mathbf{x} \right\rangle,
\end{align*}
where $s=\min\{N_x,N_1\}$.
The last approximate equality holds from that
\begin{align*}
    \left\vert \frac{1}{N_x}\sum_{i=2}^{s} \left\langle \mathbf{q}_i \hat{\mathbf{q}}_i^T\mathbf{1}_{N_x}, \mathbf{W}^{\epsilon}\mathbf{x} \right\rangle \right\rvert
    &\leq \frac{\epsilon \sqrt{H_{s-1}}}{N_x} \left\langle \mathbf{1}_{N_1}, \mathbf{W}^{\epsilon}\mathbf{x} \right\rangle\\
    &=\frac{\epsilon \sqrt{H_{s-1}}}{N_x}\sum_{j=1}^{N_1}\left\langle \mathbf{r}_j, \mathbf{x} \right\rangle\\
    &\leq \frac{\epsilon \sqrt{H_{s-1}}}{N_x} \sum_{j=1}^{N_1} \left\lVert \mathbf{x} \right\rVert,
\end{align*}
where $\mathbf{r}_j$ be the $j$-th row vector of $\mathbf{W}^{\epsilon}$ for each $j=1,\ldots,N_1$.
The first inequality holds from Proposition~\ref{prop:csum}, and the last inequality holds from orthogonality and Cauchy-Schwarz inequality.
Then by Proposition~\ref{prop:csum} it follows that 
\begin{align*}
    \frac{1}{N_x} \left\langle\mathbf{q}_1 \hat{\mathbf{q}}_1^T\mathbf{1}_{N_x}, \mathbf{W}^{\epsilon}\mathbf{x} \right\rangle
    &= C\left( N_1\frac{1}{N_1}\left\langle \mathbf{1}_{N_1}, \mathbf{W}^{\epsilon}\mathbf{x} \right\rangle
    + (\epsilon^2 + \epsilon N_x) \left\langle \mathbf{e}_1, \mathbf{W}^{\epsilon}\mathbf{x}\right\rangle \right)\\
    &\simeq \sqrt{\frac{N_1}{N_x}}\frac{1}{N_1}\left\langle \mathbf{1}_{N_1}, \mathbf{W}^{\epsilon}\mathbf{x} \right\rangle,
\end{align*}
where $C = \displaystyle\frac{N_x+\epsilon}{N_x\sqrt{\epsilon^2 + 2\epsilon +N_1}\sqrt{\epsilon^2 + 2\epsilon +N_x}}$.
\end{proof}

\begin{corollary}\label{cor1}
Given that $\epsilon$ is sufficiently small. Then the angle $\theta_1$ between the one vector $\mathbf{1}$ and $\mathbf{x}$ in $\mathbb{R}^{N_x}$ is nearly identical to the angle $\theta_2$ between the one vector $\mathbf{1}$ and $\mathbf{W}^{\epsilon}\mathbf{x}$ in $\mathbb{R}^{N_1}$.
\end{corollary}

\begin{proof}
Note that the orthogonality of $\mathbf{W}^{\epsilon}$ implies that 
$\left\lVert\mathbf{W}^{\epsilon}\mathbf{x}\right\rVert = \left\lVert\mathbf{x}\right\rVert$.
Therefore, by Theorem~\ref{thm2} one can see that 
 \begin{align*}
\cos{\theta}_2= \frac{\left\langle \mathbf{W}^{\epsilon}\mathbf{x}, \mathbf{1}_{N_1} \right\rangle}{\left\lVert\mathbf{W}^{\epsilon}\mathbf{x}\right\rVert\left\lVert\mathbf{1}_{N_1}\right\rVert} 
\simeq\sqrt{\frac{N_x}{N_1}}\frac{\left\langle \mathbf{x}, \mathbf{1}_{N_x} \right\rangle}{\left\lVert\mathbf{x}\right\rVert\left\lVert\mathbf{1}_{N_1}\right\rVert}
= \cos{\theta}_1.
\end{align*}
\end{proof}

Theorem~\ref{thm2} states that the average of $\mathbf{x}$ is linearly preserved in $\mathbf{W}^{\epsilon}\mathbf{x}$, where the proportion is almost $\sqrt{\frac{N_1}{N_x}}$. As demonstrated in Figure~\ref{fig:relu}, the average of $\mathbf{W}^{\epsilon}\mathbf{x}$ is approximately $0.35$, which is almost equal to the product of $\mathbf{x}$'s average of $0.5$ and $\sqrt{\frac{N_1}{N_x}}=\sqrt{\frac{100}{200}}\approx 0.7$. Corollary~\ref{cor1} states that the angle between $\mathbf{x}$ and the one vector is preserved in the angle between $\mathbf{W}^{\epsilon}\mathbf{x}$ and the one vector. 
If the entries of $\mathbf{x}$ are all positive, then the angle between the one vector and $\mathbf{x}$ is acute. Corollary~\ref{cor1} implies that the orthant in which $\mathbf{W}^{\epsilon}\mathbf{x}$ resides will predominantly be composed of positive values.

Now we consider a deep network of depth $\ell$ with linear activation function and zero bias. According to the definition of the proposed initial weight matrix, the following equation is satisfied.
\begin{align*}
\mathbf{y} = \mathbf{W}^{\epsilon}_{N_{\ell}\times N_{\ell -1}}\cdots \mathbf{W}^{\epsilon}_{N_{1}\times N_{x}}\mathbf{x}
=\mathbf{W}^{\epsilon}_{N_{\ell}\times N_{x}}\mathbf{x}.
\end{align*}
It means that regardless of the network's depth, $\mathbf{y}$ satisfies Theorem~\ref{thm2} and Corollary~\ref{cor1}, provided that $\epsilon$ is sufficiently small.

\begin{figure*}[t!]
\centering
\begin{subfigure}{0.33\linewidth}
\includegraphics[width=\linewidth]{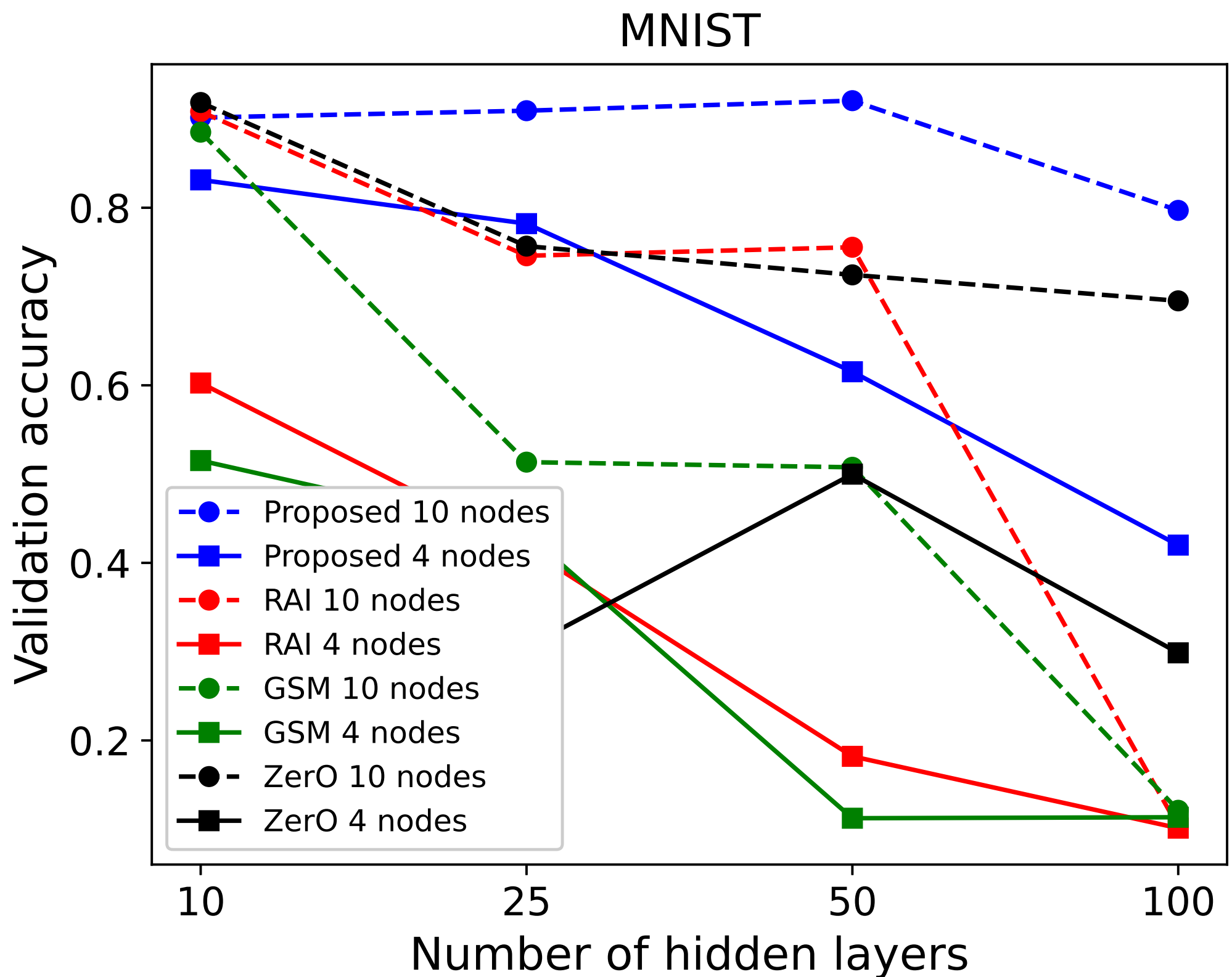}
\caption{MNIST}
\end{subfigure}
\hspace{1em}%
\begin{subfigure}{0.33\linewidth}
\includegraphics[width=\linewidth]{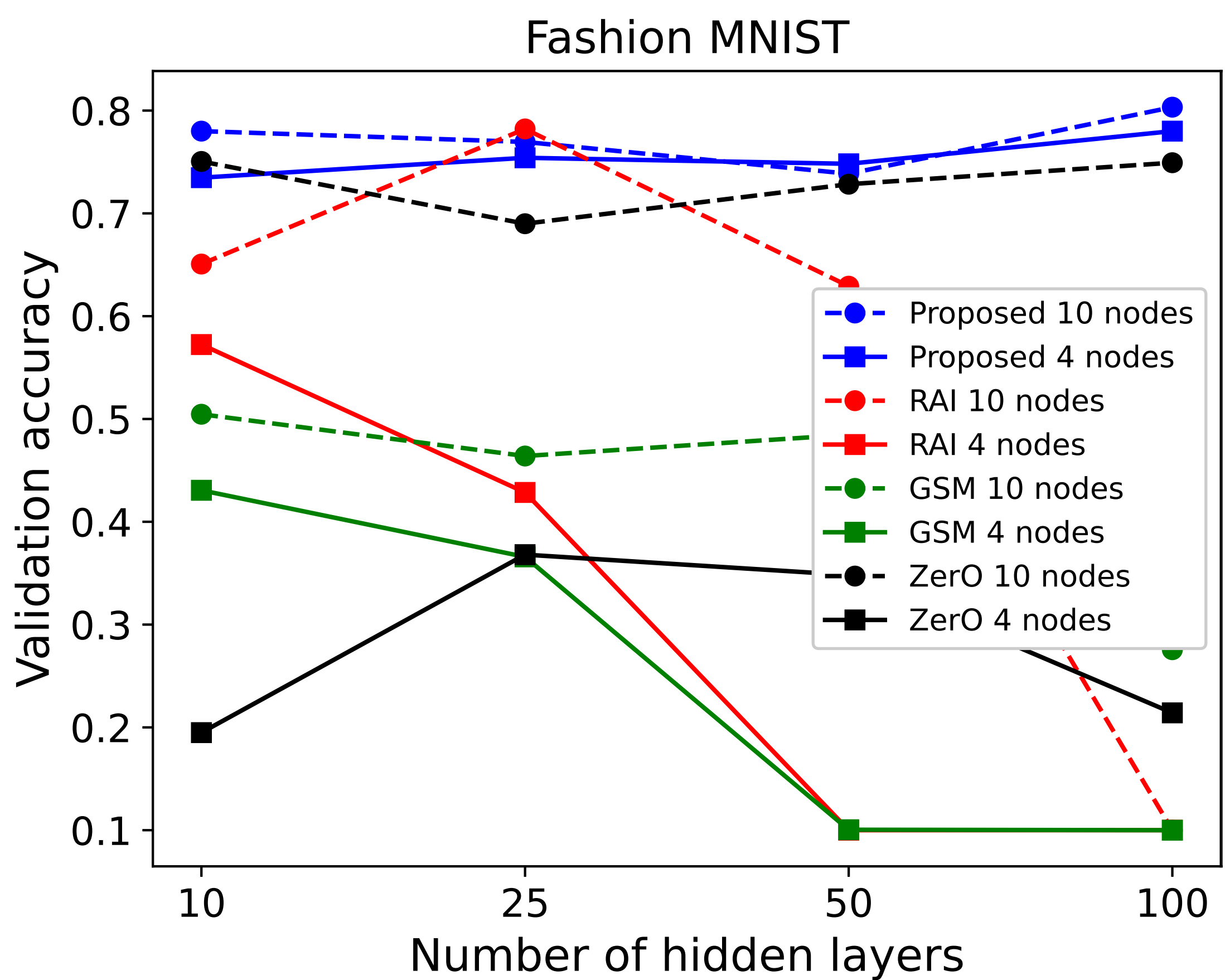}
\caption{Fashion MNIST}
\end{subfigure}
\\

\begin{subfigure}{0.33\linewidth}
\includegraphics[width=\linewidth]{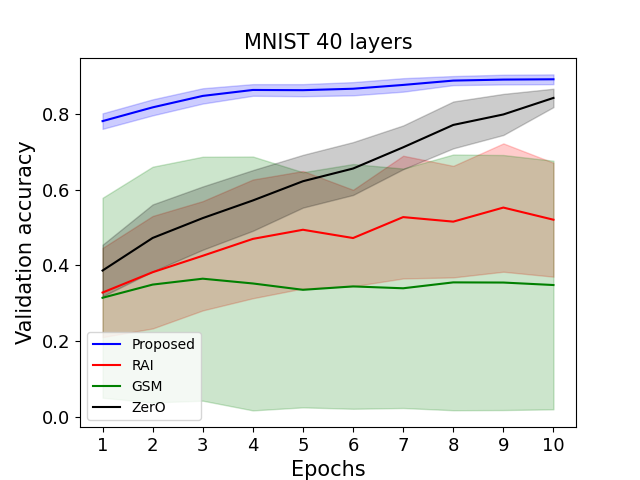}
\caption{MNIST 40 hidden layers}\label{fig:binarySkeleton}
\end{subfigure}
\hfill
\begin{subfigure}{0.33\linewidth}
\includegraphics[width=\linewidth]{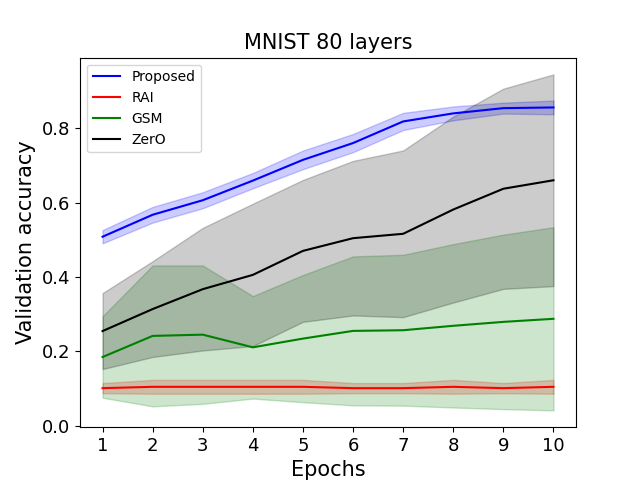}
\caption{MNIST 80 hidden layers}
\end{subfigure}
\hfill
\begin{subfigure}{0.33\linewidth}
\includegraphics[width=\linewidth]{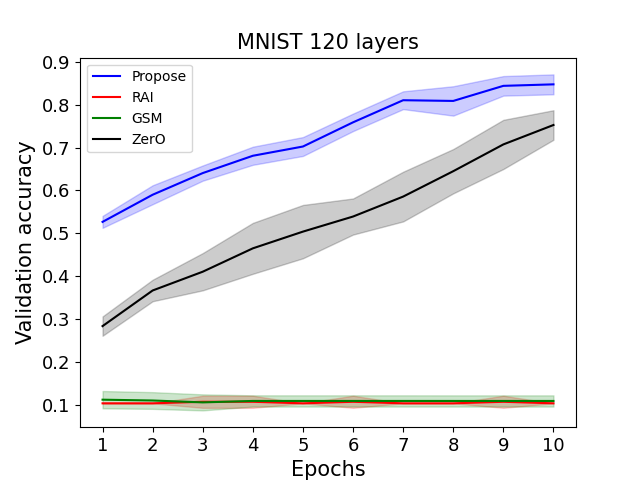}
\caption{MNIST 120 hidden layers }
\end{subfigure}
\caption{Validation accuracy for FFNNs with ReLU activation is presented across varying depths to explore layer independence. (a) and (b) investigate networks where all hidden layers maintain the same dimension. (c), (d), and (e) investigate networks consisting of a layer with 10 nodes and a layer with 6 nodes, repeated throughout the structure.}
\label{fig:notsame}
\end{figure*}

\smallskip

\section{Experimental results}\label{sec:experiments}
In Section~\ref{subsec:Benchmark}, we describe the experimental environment, benchmark datasets, and the methods 
compared in this paper.
In Section~\ref{subsec:prior}, we introduce existing methods that effectively prevent the dying ReLU problem in deep 
and narrow FFNNs with ReLU activation function, and we present the settings used for each method in 
our experiments.
Section~\ref{subsec:various_settings} presents the results from experiments across diverse dataset sizes and network architectures. 
In Section~\ref{subsec:lay_indep}, we experiment to investigate the trainability of networks at various depths.
In Section~\ref{subsec:node_indep}, we conduct experiments to determine the feasibility of training on extremely narrow networks. Finally, Section~\ref{subsec:activation_indep} investigates the trainability of networks at various activation functions.

\subsection{Experimental Settings}\label{subsec:Benchmark}
Section~4 analyzed a range of weight initialization methods via a comprehensive experimental approach.
We referred to our method as \textquotedblleft Proposed".
As other methods, we used Xavier
initialization, He initialization, ZerO initialization, Identity initialization, Orthogonal initialization, RAI, and GSM.
To assess the effectiveness of the proposed weight initialization method, 
we conducted experiments on MNIST, Fashion MNIST, Wine Quality dataset, and Iris. 
And $15$\% of the total dataset served as the validation dataset.
We trained the network using cross-entropy loss implemented the neural network in Python with Tensorflow, and trained the neural network on a computer with GPU RTX TITAN. And we used Adam optimizer with a learning rate of $0.001$ and a batch size of $100$.
Empirically setting $\epsilon$ to 0.1, we configured all activation function parameters to their TensorFlow default values. For each dataset, we maintained consistent hyperparameter settings for all experiments. Each experiment was repeated ten times with random seeds.

\subsection{Prior Weight Initialization Method for FFNNs}\label{subsec:prior}
We briefly review Randomized Asymmetric Initialization (RAI)~\cite{reluinit1} and Gaussian Submatrix Initialization (GSM)~\cite{reluinit2}, both of which are designed to enhance the training of deep and narrow feedforward neural networks. 
And we also present the settings for each of them.
The RAI weight initialization method is a technique used to address the \textquotedblleft dying ReLU" problem in 
deep FFNNs. It involves the creation of an initialization matrix with random values drawn from an asymmetric probability distribution. 
In this paper, we employed Beta(2,1) probability distribution. 
And standard deviation of the weight matrix was 
$-\frac{2\sqrt{2}}{3\sqrt{\pi}} + \sqrt{1+\frac{8}{9\pi}}\approx 0.6007$ adopting a setting similar to that in \cite{reluinit1}. The GSM is a weight initialization method for ReLU layers that ensures perfect dynamical isometry. In this paper, we constructed the submatrix using the He initialization method.

\subsection{Experiments in Various Settings}\label{subsec:various_settings}
In this section, we conducted experiments using the MNIST and Fashion MNIST
datasets. We compared various initialization methods while varying the dataset size with FFNNs with ReLU activation function.
As shown in Table~\ref{table:nohidden}, we measured the validation accuracy in various settings at $10$ and $100$ epochs. The term \textquoteleft $k$ samples per class' indicates that each class consists of $k$ number of samples. We conducted experiments for $1,2,4$ samples per class and the entire dataset. 
\begin{figure*}[h!]
\begin{subfigure}{0.26\linewidth}
 \includegraphics[width=\linewidth]{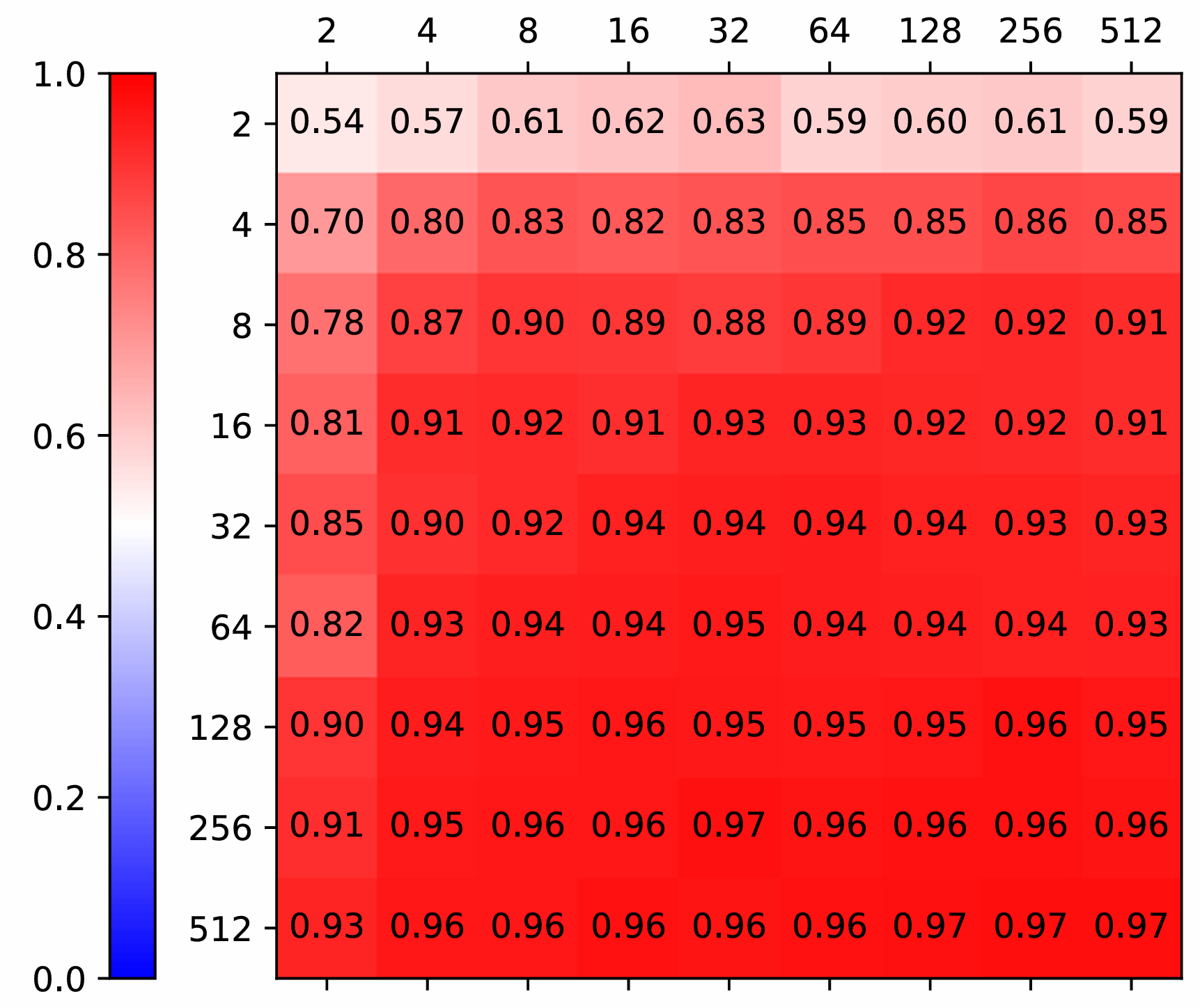}
\caption{Proposed}
\end{subfigure}
\hfill
\begin{subfigure}{0.23\linewidth}
\includegraphics[width=\linewidth]{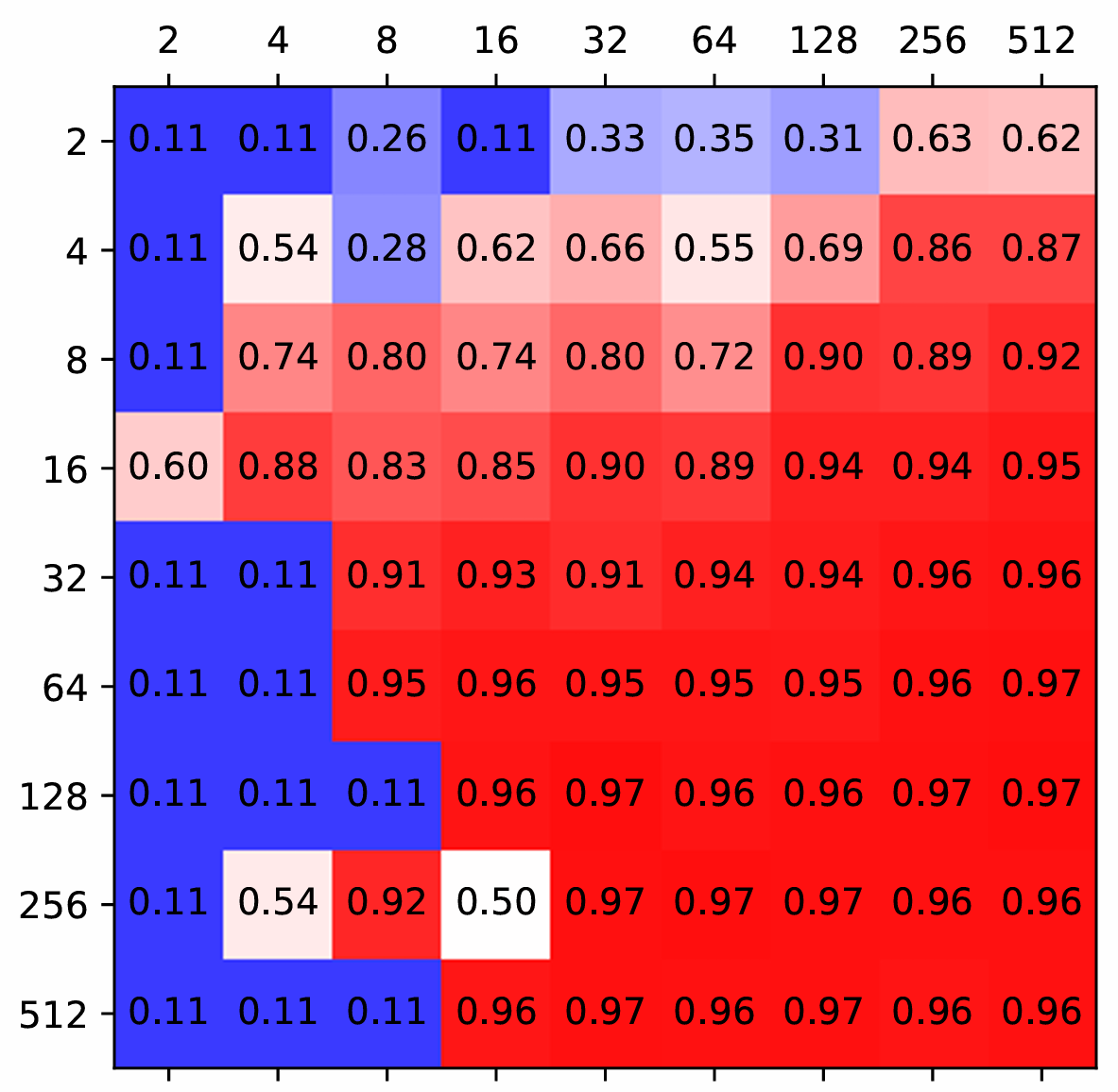}
\caption{Orthogonal}
\end{subfigure}
\hfill
\begin{subfigure}{0.23\linewidth}
\includegraphics[width=\linewidth]{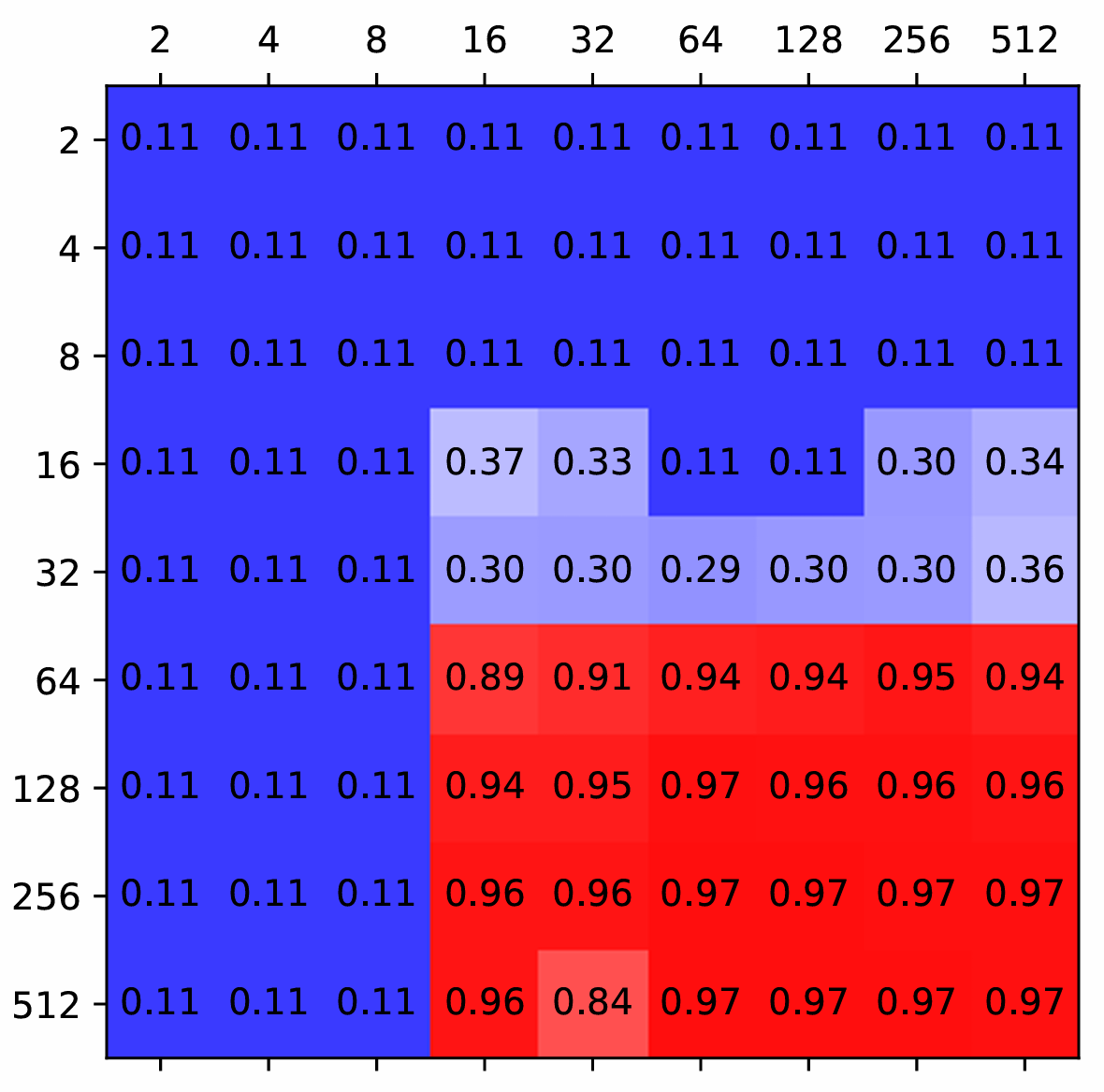}
\caption{Identity}
\end{subfigure}
\hfill
\begin{subfigure}{0.23\linewidth}
\includegraphics[width=\linewidth]{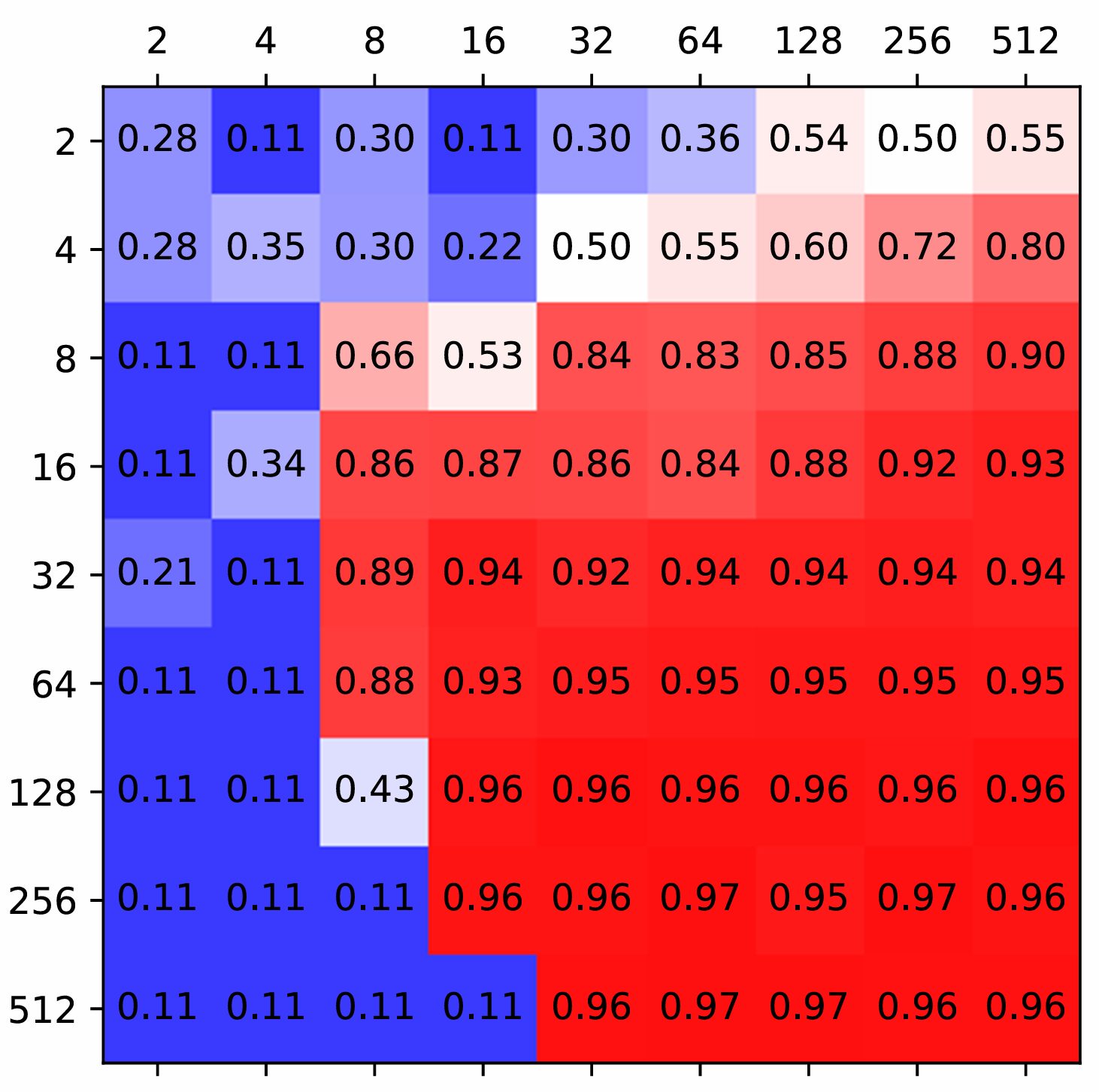}
\caption{Xavier}
\end{subfigure}
\\

\begin{subfigure}{0.23\linewidth}
 \includegraphics[width=\linewidth]{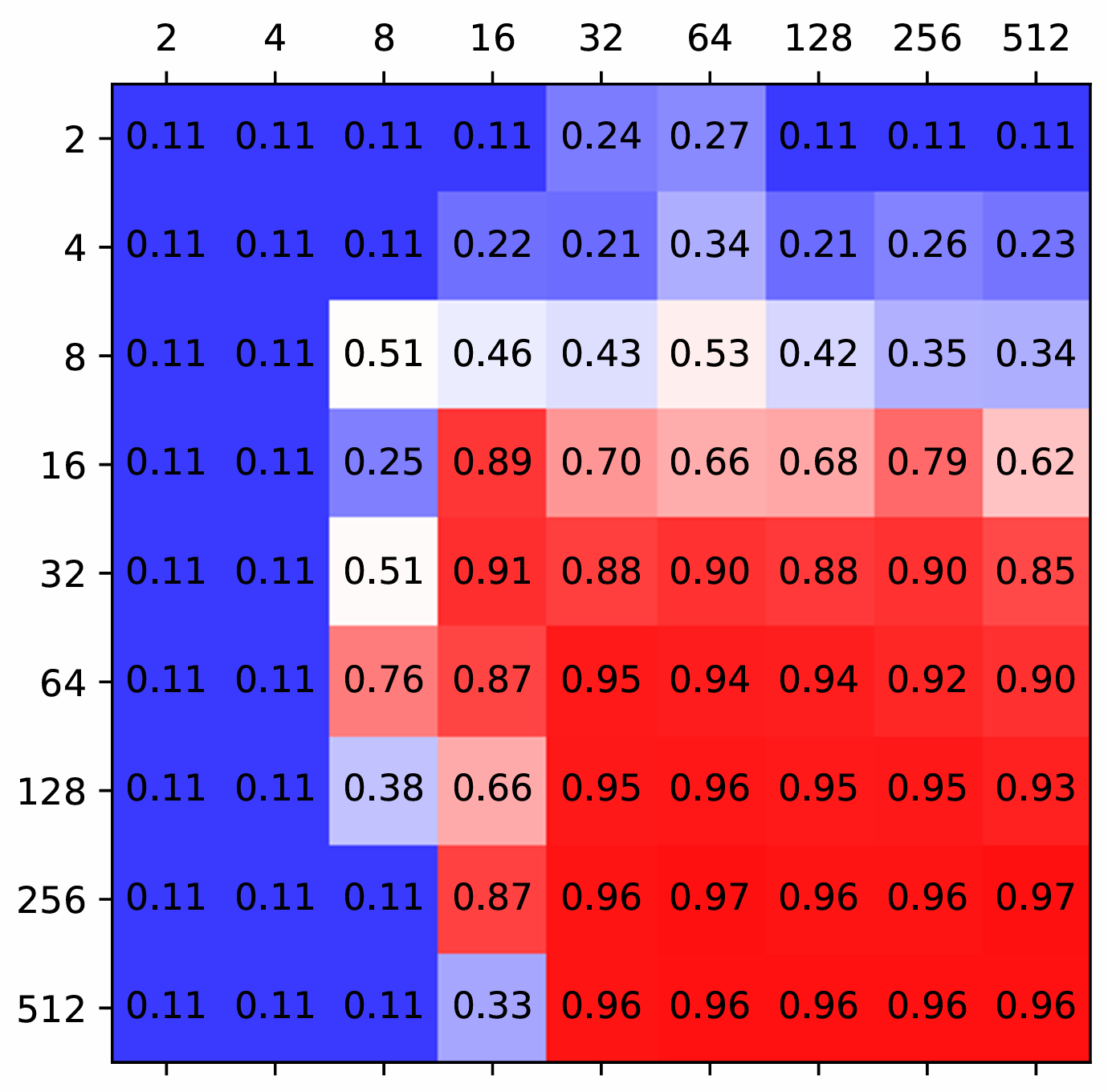}
\caption{He}
\end{subfigure}
\hfill
\begin{subfigure}{0.23\linewidth}
\includegraphics[width=\linewidth]{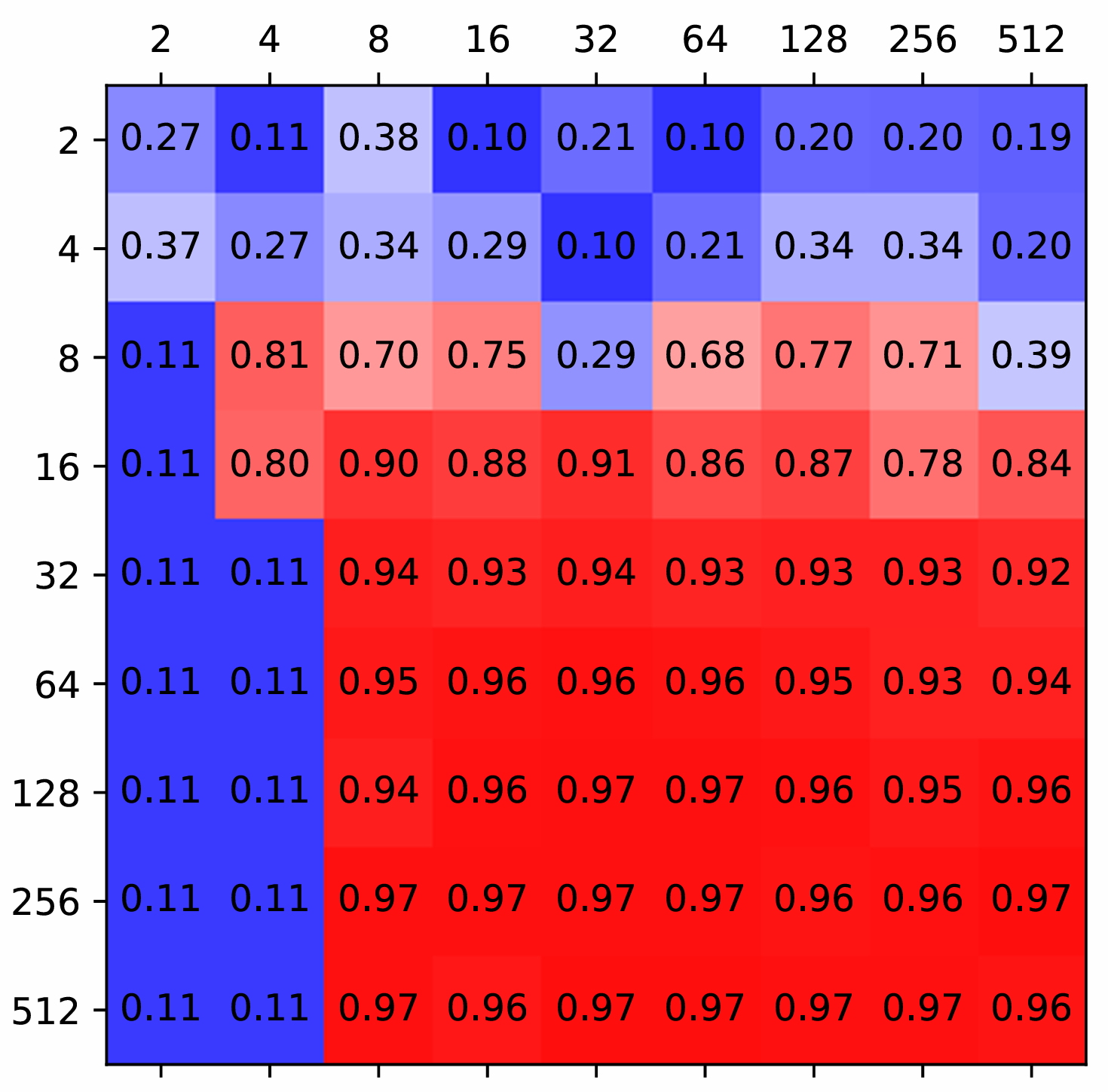}
\caption{RAI}
\end{subfigure}
\hfill
\begin{subfigure}{0.23\linewidth}
\includegraphics[width=\linewidth]{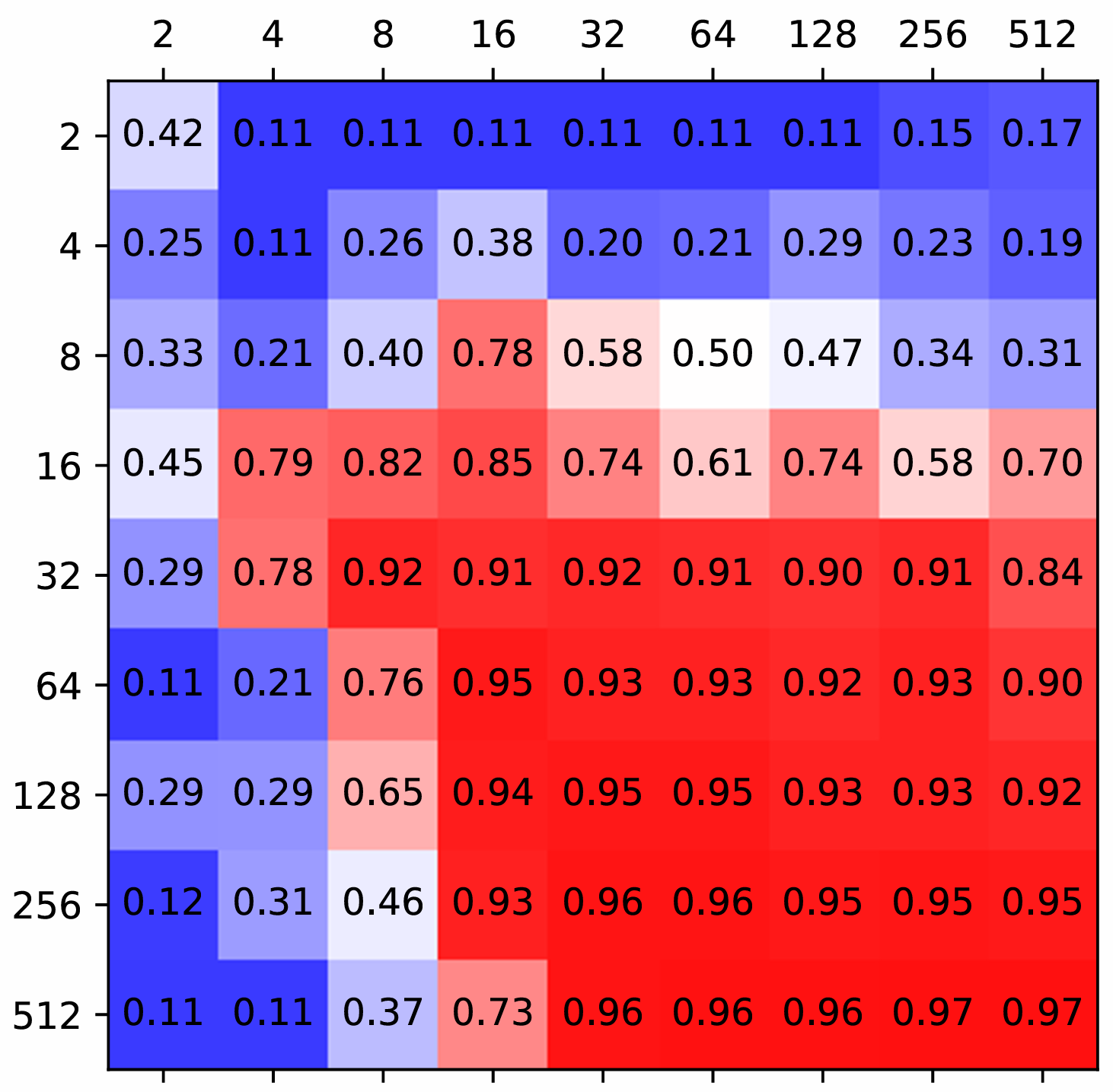}
\caption{GSM}
\end{subfigure}
\hfill
\begin{subfigure}{0.23\linewidth}
\includegraphics[width=\linewidth]{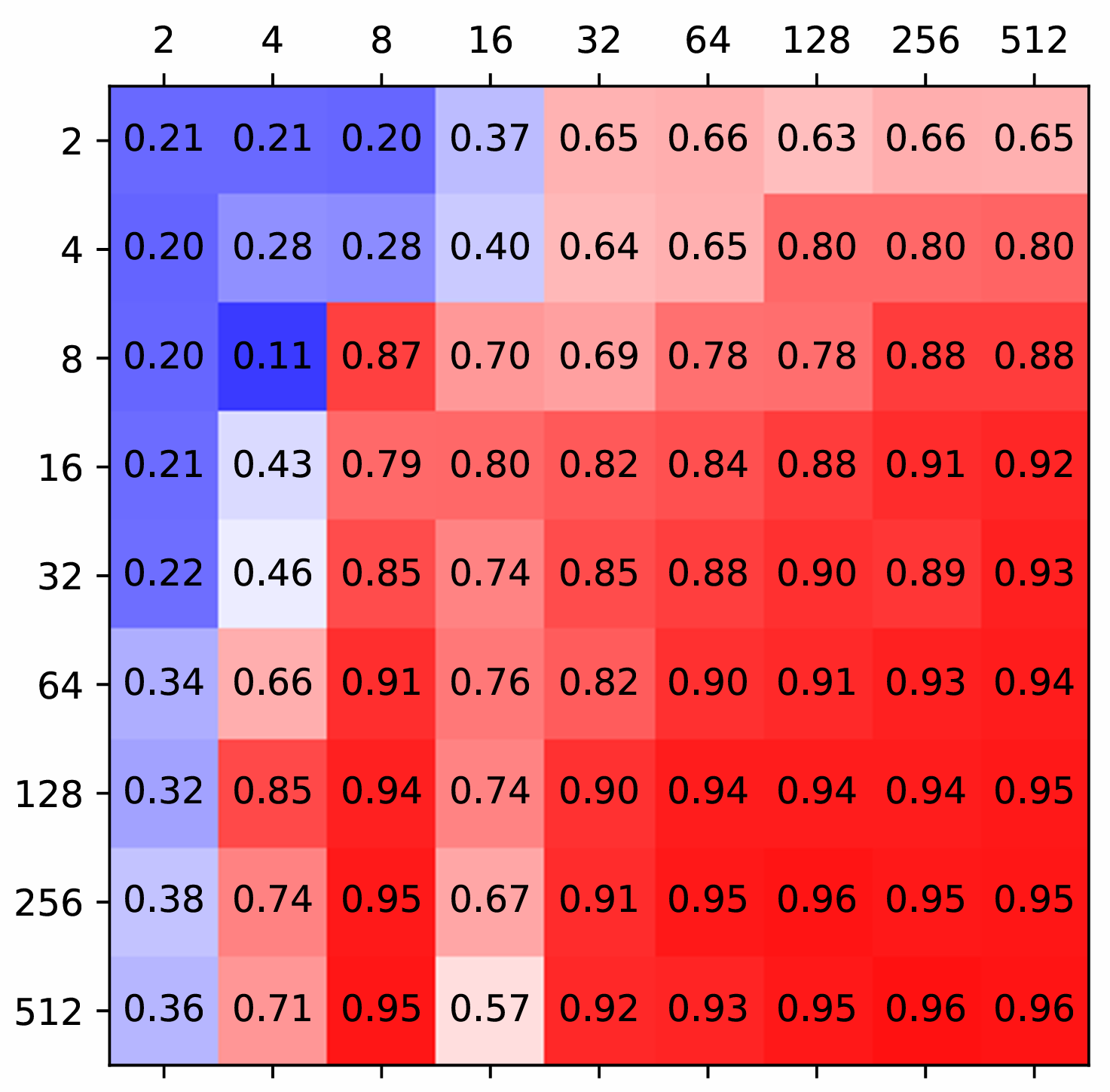}
\caption{ZerO}
\end{subfigure}
\caption{A validation accuracy is presented for FFNNs with two hidden layers and ReLU activation function.
The $y$-axis (resp. $x$-axis) presents the number of nodes in the first (resp. second) hidden layer.
Each is trained on MNIST dataset for 10 epochs.}
 \label{fig:nodesheat1}
\end{figure*}
In the table, (0), (512), and (16) denote the number of nodes in a single hidden layer of FFNNs with ReLU activation function.
In detail, $(0)$ signifies an FFNN without hidden layers; $(512)$ corresponds to one with a single hidden layer of 512 nodes; and $(16)$, one with a single hidden layer of 16 nodes. 
With no hidden layers, the proposed method consistently achieved higher validation accuracy at 10 epochs, irrespective of the dataset size. 
Both identity initialization and Zero initialization also demonstrated high validation accuracy. Zero initialization is an identity matrix when the number of rows in the weight matrix is less than the number of columns. These three weight initialization methods outperformed random weight initialization in networks without hidden layers. Furthermore, even for small values of $k$, all three methods exhibited good performance. When $k$ equals $1$, the decline in accuracy observed for the proposed method on the FMNIST dataset at $100$ epochs was attributed to overfitting, possibly due to excessively rapid convergence.
An FFNN with $512$ nodes can be considered a wide FFNN, whereas an FFNN with $16$ nodes is relatively narrow.
We conducted comparative experiments on these two FFNNs to assess how independent our proposed method is regarding the number of nodes.
Generally, in narrow networks with fewer nodes, learning is less effective compared to wider networks with a larger number of nodes.
However, the proposed weight initialization exhibited significantly higher validation accuracy even with 16 nodes, surpassing other weight initialization methods.
Contrastingly, in networks with $512$ nodes, it exhibited validation accuracy similar to other weight initialization methods. 
The reason is that in the narrow network, the dying ReLU problem is particularly detrimental to network training. To demonstrate that the proposed method is independent of both network depth and width, more diverse experiments are needed.

\begin{figure*}[h]
\begin{subfigure}{0.26\linewidth}
 \includegraphics[width=\linewidth]{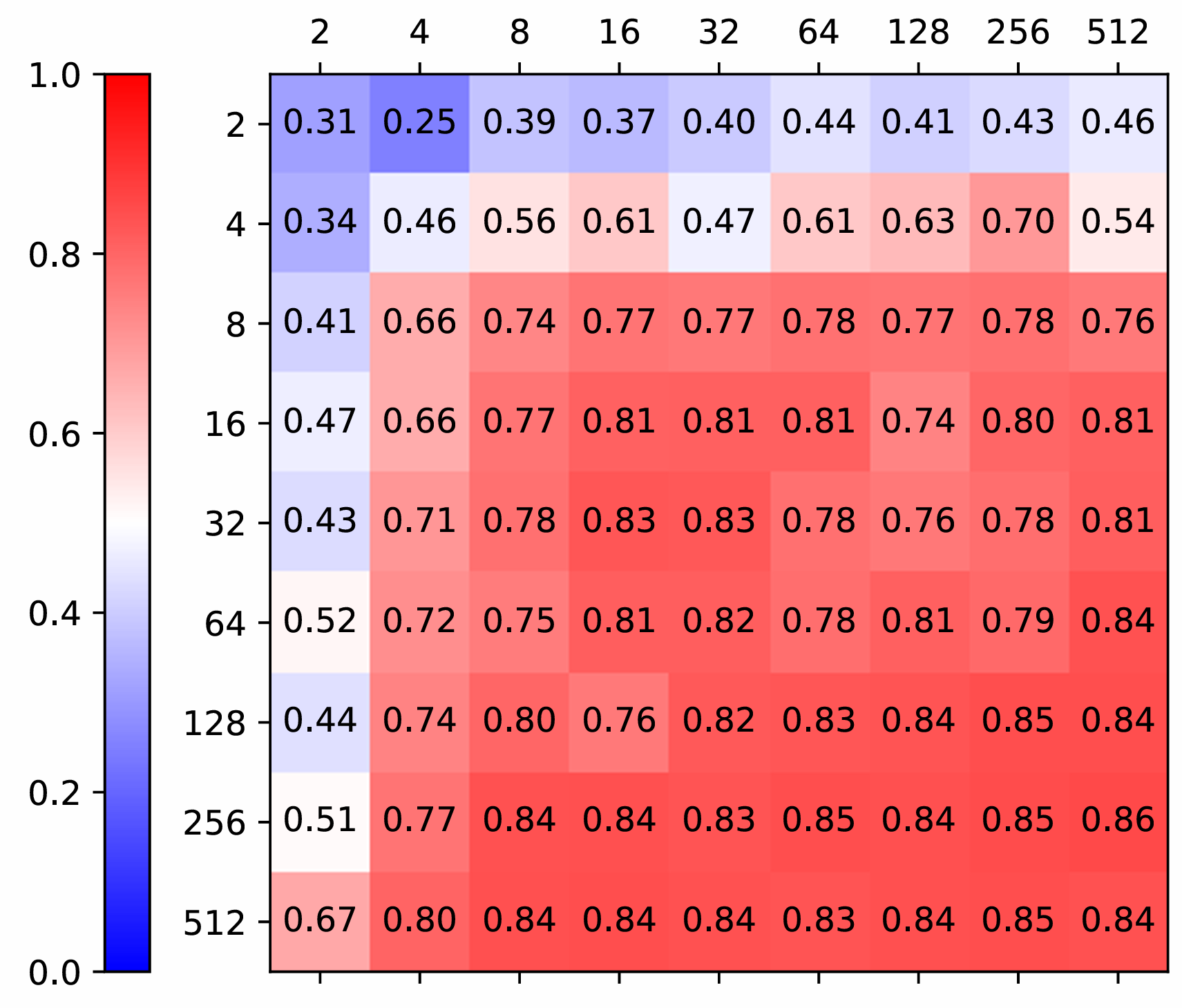}
\caption{Proposed}
\end{subfigure}
\hfill
\begin{subfigure}{0.23\linewidth}
\includegraphics[width=\linewidth]{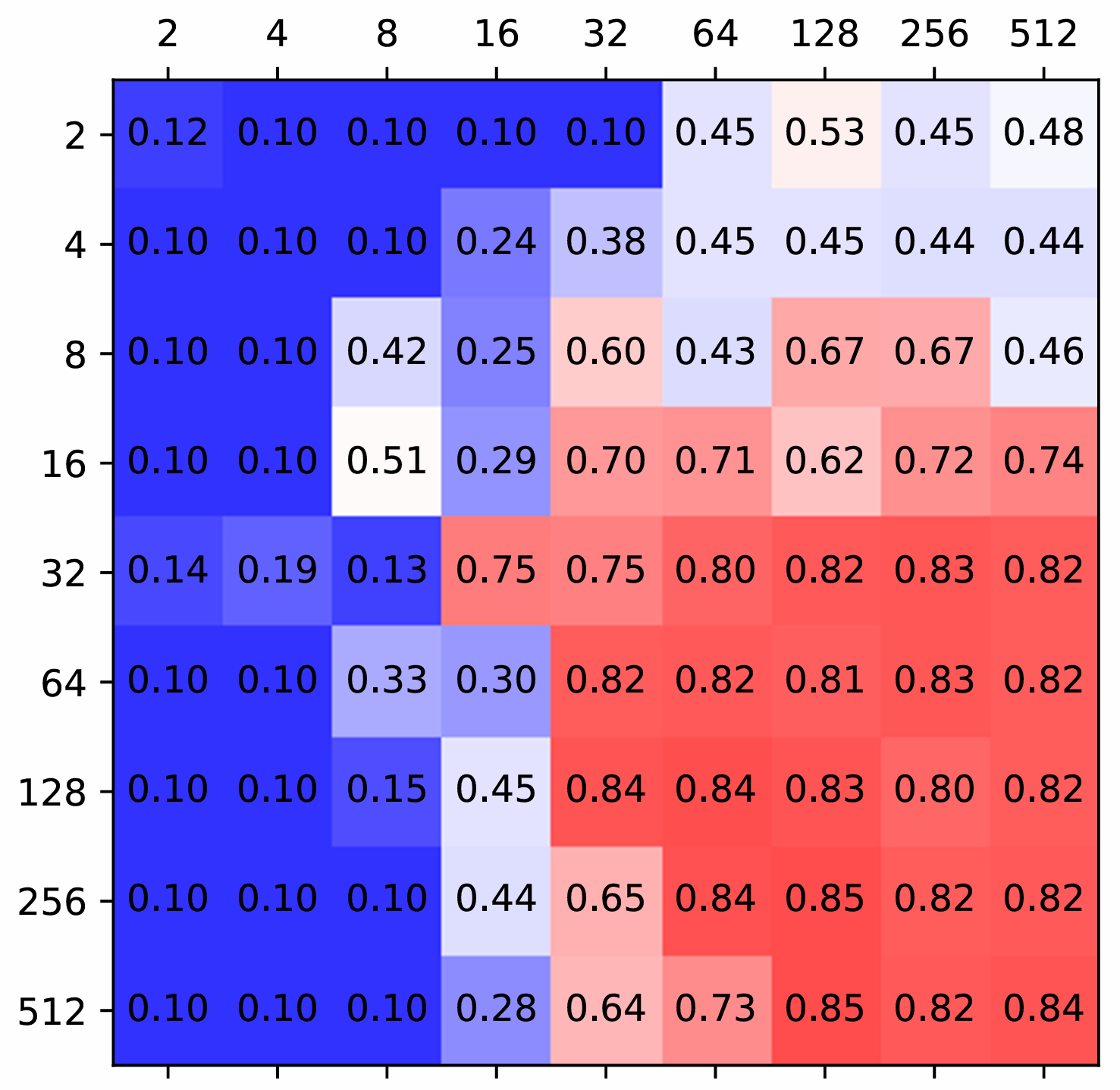}
\caption{Orthogonal}
\end{subfigure}
\hfill
\begin{subfigure}{0.23\linewidth}
\includegraphics[width=\linewidth]{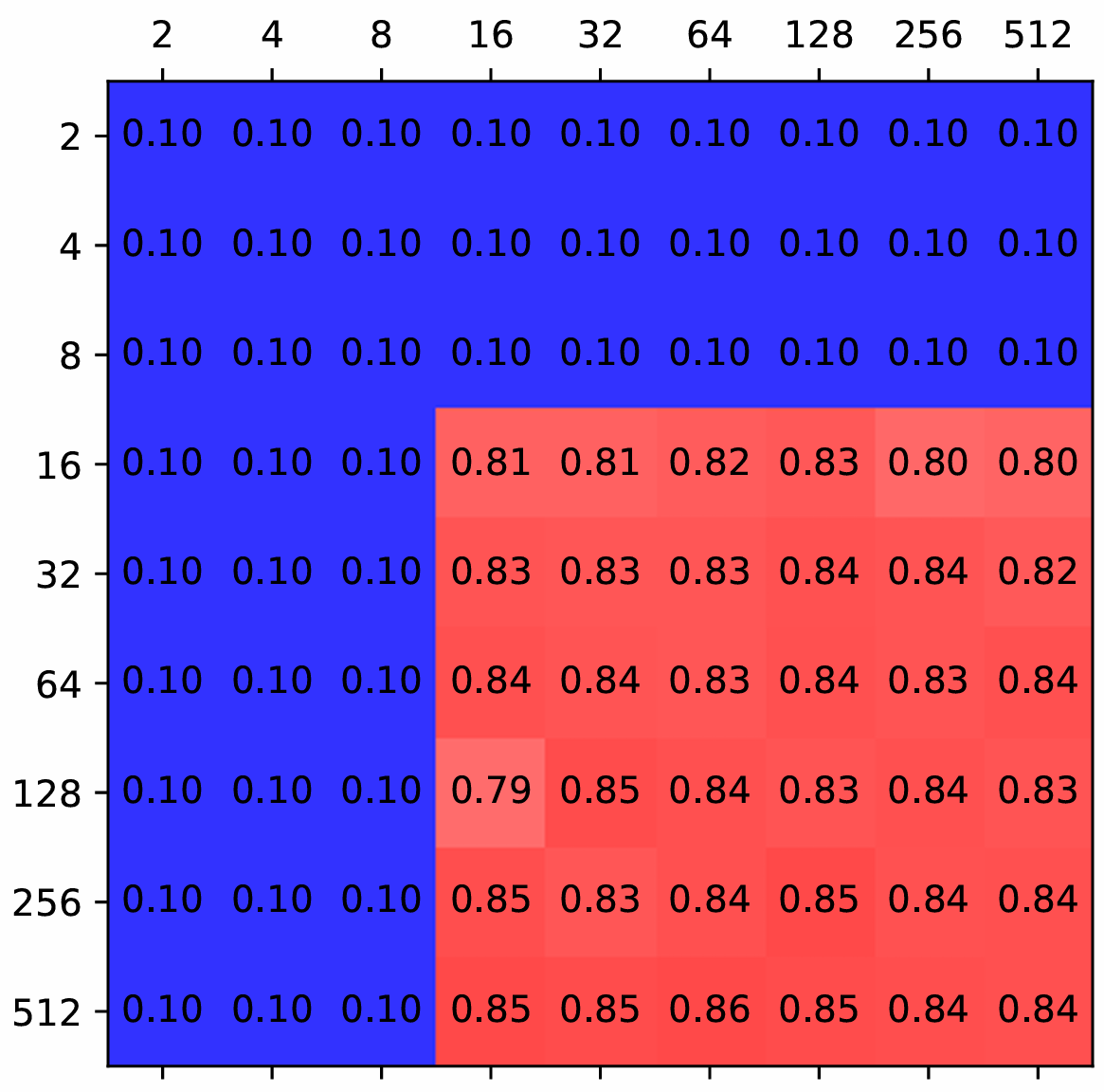}
\caption{Identity}
\end{subfigure}
\hfill
\begin{subfigure}{0.23\linewidth}
\includegraphics[width=\linewidth]{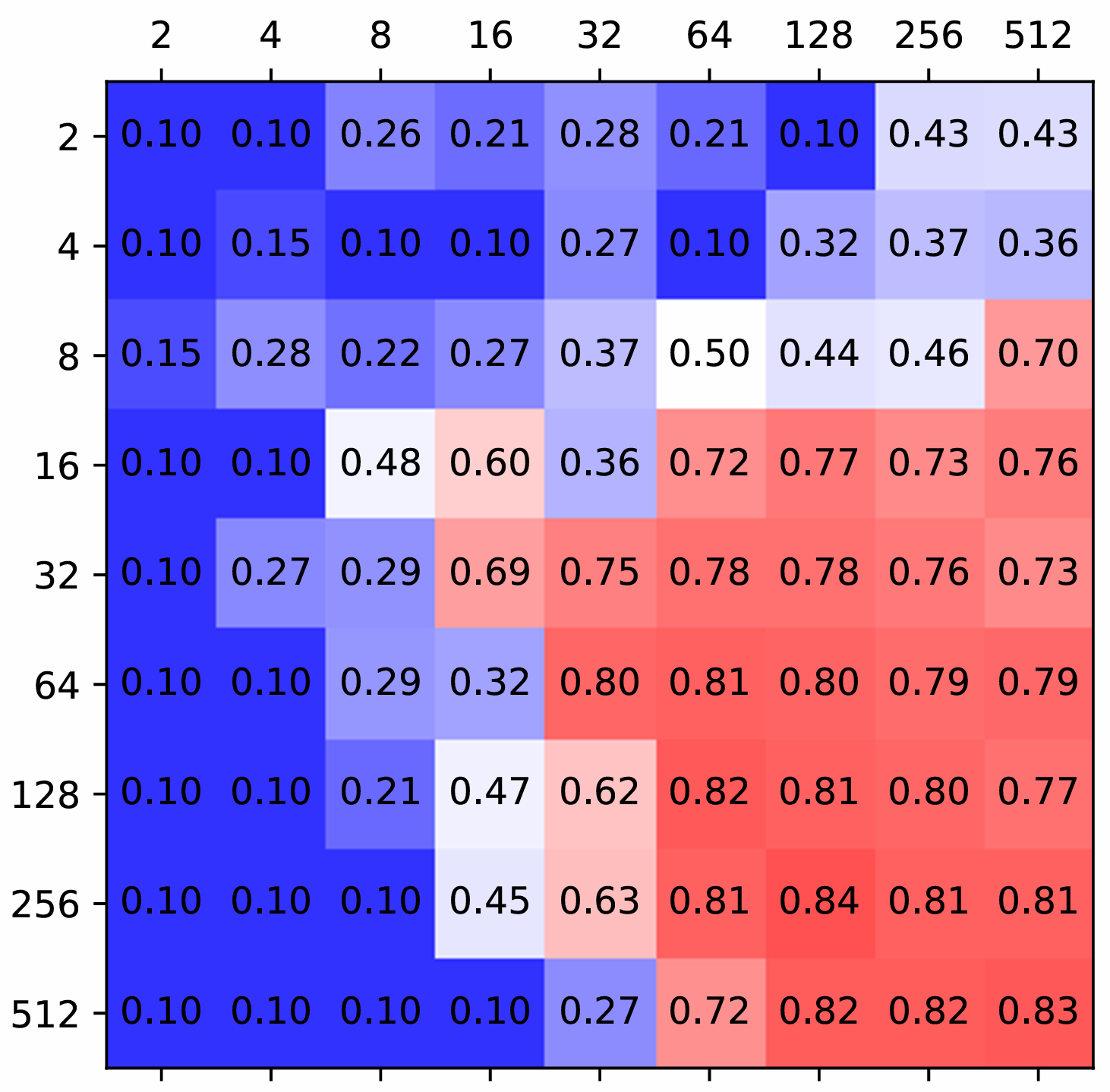}
\caption{Xavier}
\end{subfigure}
\\

\begin{subfigure}{0.23\linewidth}
 \includegraphics[width=\linewidth]{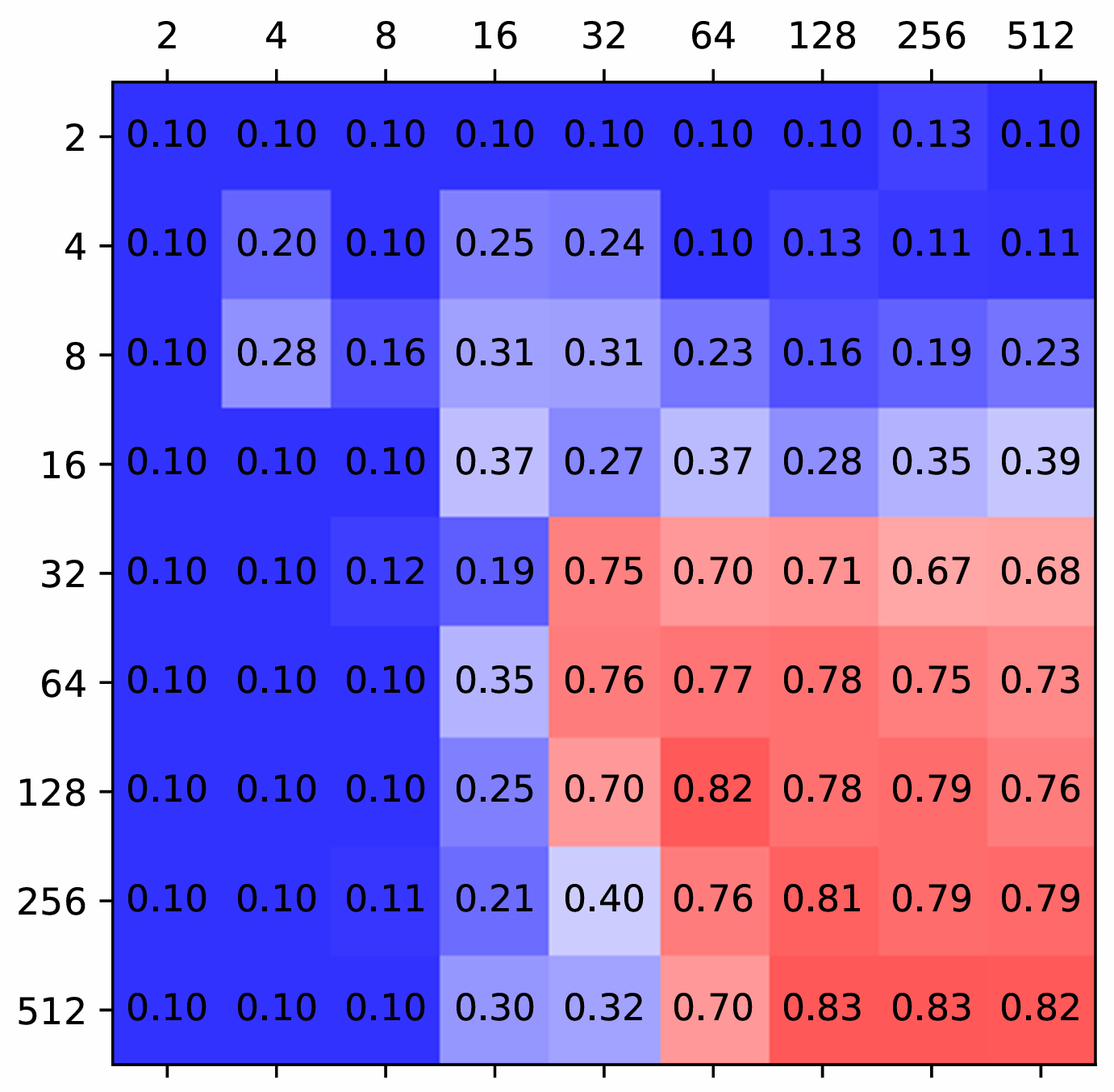}
\caption{He}
\end{subfigure}
\hfill
\begin{subfigure}{0.23\linewidth}
\includegraphics[width=\linewidth]{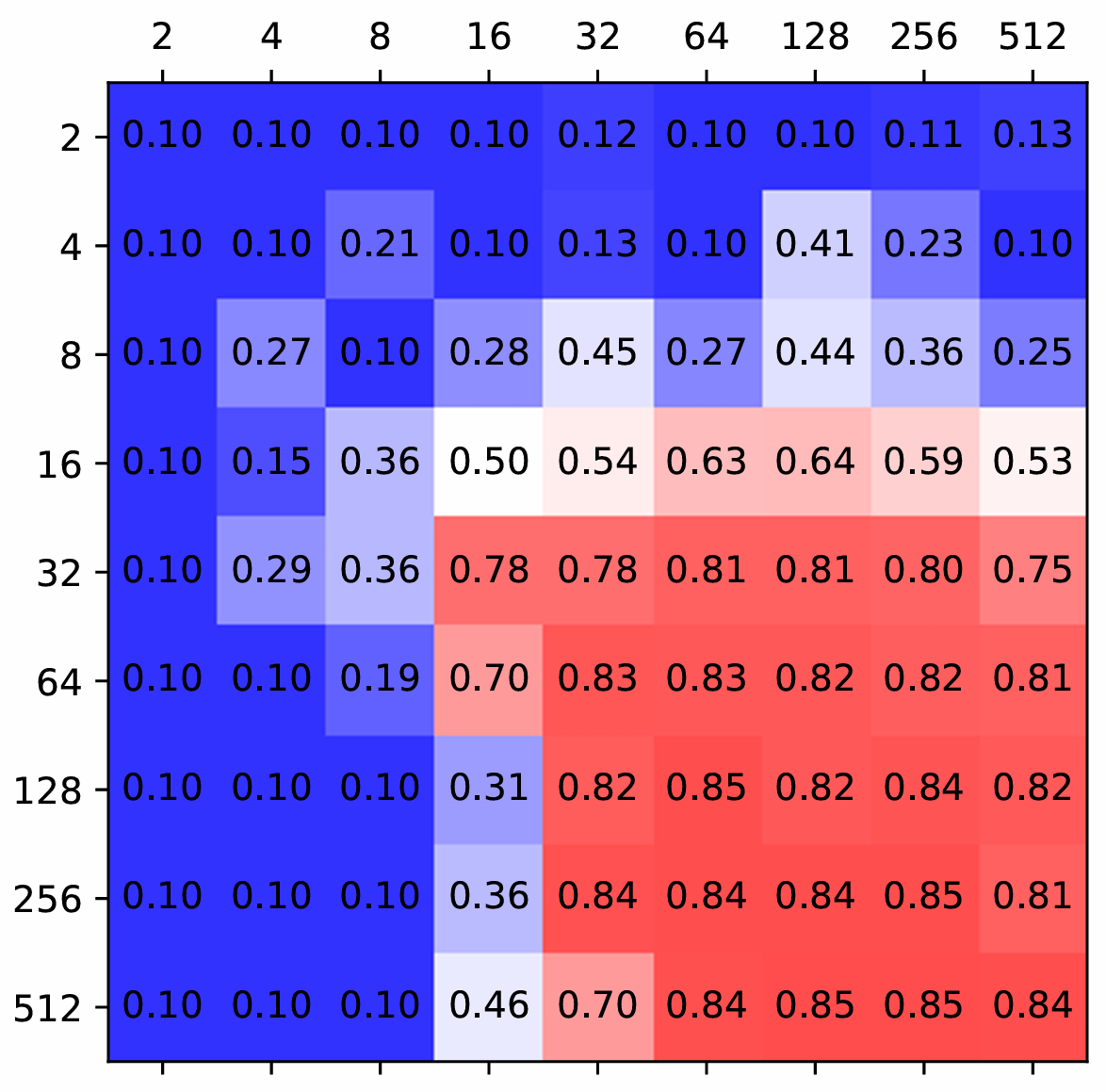}
\caption{RAI}
\end{subfigure}
\hfill
\begin{subfigure}{0.23\linewidth}
\includegraphics[width=\linewidth]{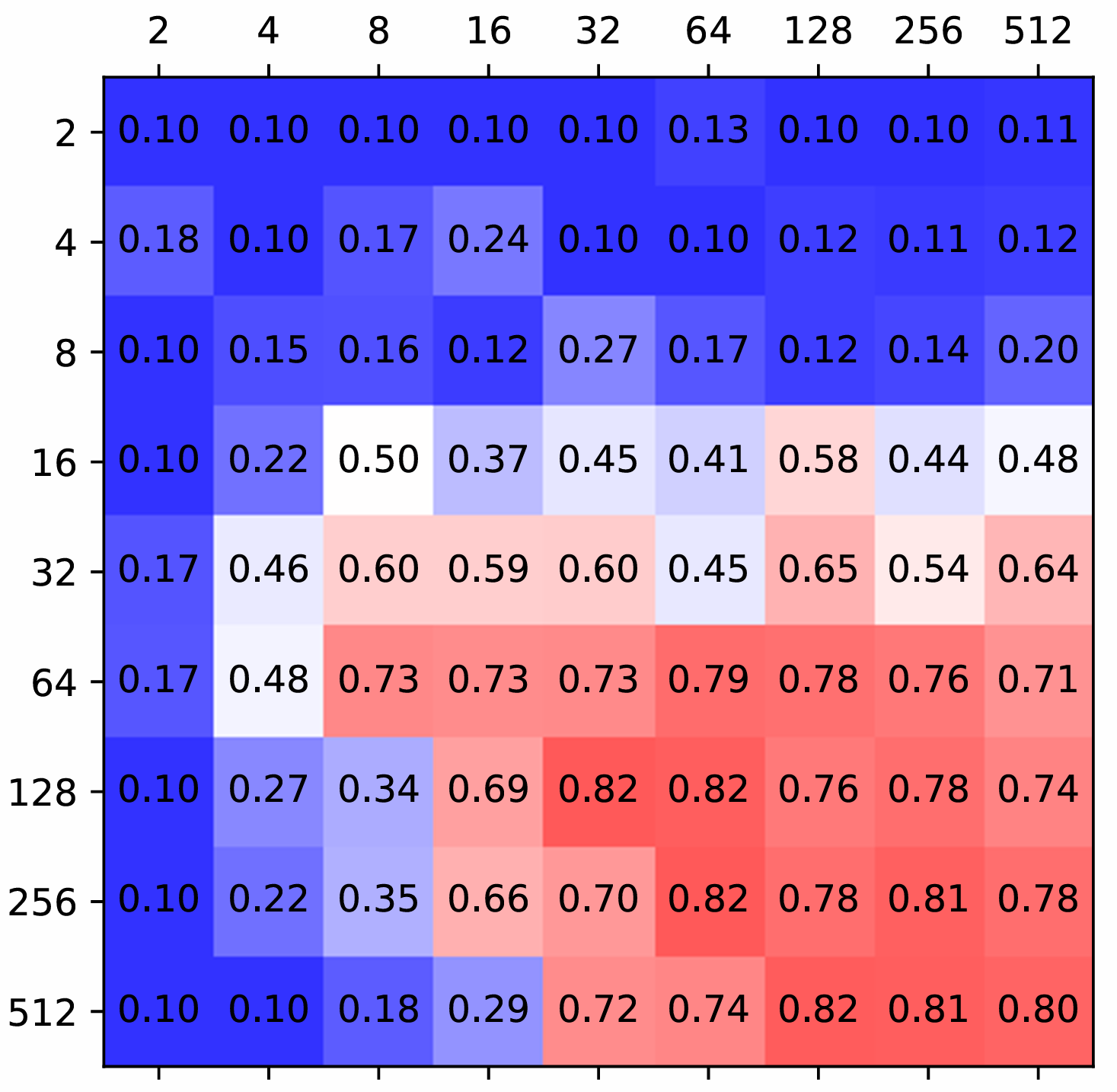}
\caption{GSM}
\end{subfigure}
\hfill
\begin{subfigure}{0.23\linewidth}
\includegraphics[width=\linewidth]{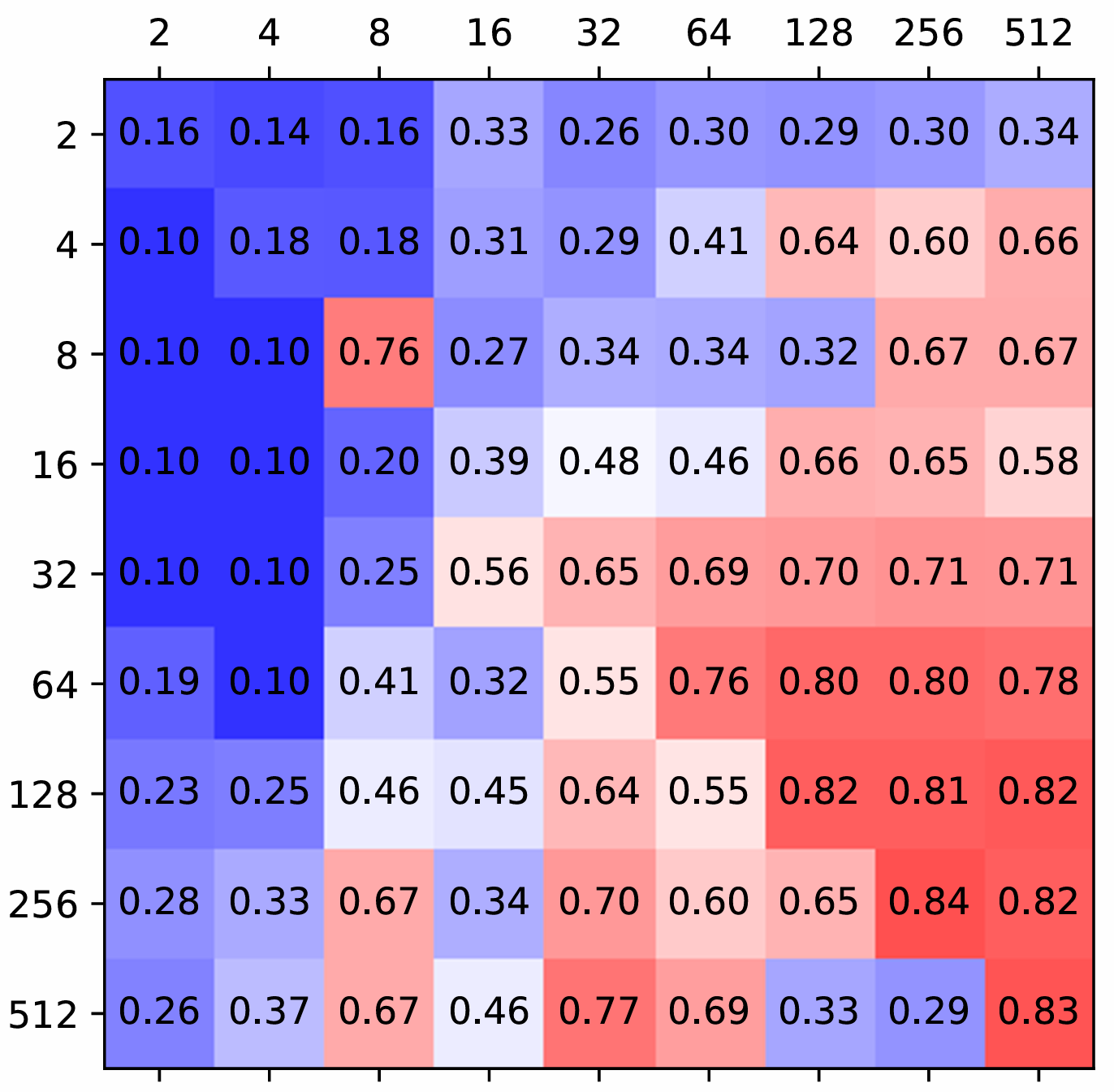}
\caption{ZerO}
\end{subfigure}
\caption{A validation accuracy is presented for FFNNs with two hidden layers and ReLU activation function.
The $y$-axis (resp. $x$-axis) presents the number of nodes in the first (resp. second) hidden layer. Each is trained on FMNIST dataset for 1 epoch.}
\label{fig:pre-processing}
\end{figure*}

\subsection{Depth Independent}\label{subsec:lay_indep}
In this section, we applied the proposed weight initialization method to investigate its effectiveness in training deep FFNNs with the ReLU activation function.
We compared the proposed initialization method with the RAI, GSM, and ZerO methods - previously studied and proven to perform well in training deep neural networks - using the MNIST and Fashion MNIST datasets. 
This experiment drew inspiration from the methods described in~\cite{reluinit2, reluinit1,zero}.

The experiments were divided into two main parts: one where the hidden layers had the same dimensionality, and the other where they had varying dimensionality. We made this division because networks in practice often exhibit varying dimensionalities across their layers.
Figure~\ref{fig:notsame} (a), (b) shows that the $x$-axis represents the number of hidden layers, while the $y$-axis represents the validation accuracy measured at 10 epochs. The label $10$ nodes indicate that all hidden layers have 10 nodes, while 4 nodes signify that each layer has 4 nodes. The experimental results of the MNIST dataset indicated that our proposed method demonstrated high validation accuracy, independent of the number of hidden layers and nodes.
Specifically, for configurations with $10$ nodes per layer, we observed that RAI and GSM were effective up to $50$ hidden layers, after which they experienced difficulty in training as the number of layers increased to $100$. ZerO initialization performed reasonably well, but when compared to our proposed method, it consistently yielded lower validation accuracy across various numbers of hidden layers. In scenarios where the network had only four nodes per layer, most initialization methods struggled due to the narrow network architecture. However, our proposed initialization method stood out by successfully enabling training even with 100 hidden layers. 
On the Fashion MNIST dataset, the experimental results also indicated that our proposed method demonstrated high validation accuracy, independent of the number of hidden layers and nodes. In contrast, zero initialization achieved high validation accuracy with $10$ nodes per layer but faced instability in narrower networks. Furthermore, RAI and GSM struggled to train networks with $50$ or $100$ hidden layers effectively.

In the experiments involving hidden layers with varying numbers of nodes, the results are depicted in Figure~\ref{fig:notsame} (c),(d),(e). The network architecture consisted of a layer with $10$ nodes and a layer with $6$ nodes, repeated throughout the structure. For instance, a network with $20$ hidden layers comprised $10$ node layers and $6$ node layers repeated $10$ times.
When the network had $40$ hidden layers, the validation accuracy of RAI and GSM across epochs exhibited significant variability. In contrast, our proposed method maintained stable validation accuracy during training. This trend continued as the number of hidden layers increased to 80, and it became evident that when the network comprised 120 hidden layers, only our proposed method and zero initialization managed to facilitate successful learning. Furthermore, we conducted simulations on two types of tabular data, each with fewer than 100 features: the Wine Quality Dataset~\cite{wine} and the Iris dataset\cite{iris}. In further experiments, we compared the proposed initialization method with the RAI, ZerO, He, and Orthogonal initialization methods. We trained on the Wine Quality Dataset~(resp. the Iris Dataset) using an FFNN configured with ReLU activation, comprising layers of $10$ nodes and layers of $6$ nodes, this configuration being repeated 60~(resp. 100) times. The experimental results indicated that the proposed initialization method achieved higher validation accuracy compared to other methods across both the Wine Quality and Iris datasets.
For the Wine Quality dataset at 200 epochs, the proposed method's validation accuracy~(58\%) surpassed those of ZerO~(50\%), Orthogonal~(40\%), RAI~(40\%), and He~(40\%). In the case of the Iris dataset, the proposed method rapidly achieved high accuracy by the 10th epoch and maintained it. Proposed~(94\%), ZerO~(63\%), Orthogonal~(30\%), RAI~(38\%), and He~(38\%) are validation accuracies at 100 epochs for the respective methods. 
In summary, our method demonstrated depth independence by achieving higher validation accuracy in deep networks compared to other initialization methods.

\begin{table*}[th]
\centering
\resizebox{15cm}{!}{% 
\begin{tabular}{l cc cc cc cc cc cc cc cc}
\toprule
 & \multicolumn{2}{c}{Proposed} & \multicolumn{2}{c}{Orthogonal}
 & \multicolumn{2}{c}{Xavier}& \multicolumn{2}{c}{He}
 & \multicolumn{2}{c}{Zero}& \multicolumn{2}{c}{Identity}
 & \multicolumn{2}{c}{RAI}& \multicolumn{2}{c}{GSM}\\
\cmidrule(lr){2-17} 
Datset & M&F& M&F& M&F&M&F&M&F&M&F&M&F&M&F   \\
\midrule
Tanh      & 11.1&10.0&14.3&9.9& 11.0&9.9&11.7&9.9&10.3&9.9&\textbf{27.0}&9.9&16.1&10.0&12.3&\textbf{13.6}  \\
Sigmoid   & 11.3&10.0&10.3&10.0& 10.3&9.9&11.3&9.9&10.3&10.0&10.2&9.9&10.2&9.9&10.3&10.0 \\
Selu     & \textbf{38.3}&33.3& 11.7&9.9& 10.2&9.9&9.8&9.9&33.0&\textbf{34.5}&10.4&9.9 &10.2&9.9&12.0&11.0 \\
Gelu      & \textbf{83.6}&\textbf{68.1}&65.5&10.0& 11.2&10.0&11.3&10.9&76.2&65&11.3&10.0&11.0&34.0&13.1&34.4 \\
Relu    & \textbf{86.7}&\textbf{76.5}&11.3&10.0& 11.3&10.1&11.3&9.9&82.9&69.4&11.3&10.0&11.3&10.0&11.3&8.6 \\
\bottomrule 
\end{tabular}}
\caption{A validation accuracy is presented for FFNNs with various activation functions.
The FFNN comprises $120$ hidden layers with a layer of $10$ nodes and a layer of $6$ nodes repeated $60$ times each. Each is trained on MNIST~(M) and FMNIST~(F) datasets for $10$ epochs. Best results are marked in bold.}
\label{table:activation}
\end{table*}

\subsection{Width Independent}\label{subsec:node_indep}
In this section, we employed the proposed weight initialization method to assess its effectiveness in training feedforward neural networks (FFNNs) with ReLU activation function, emphasizing its independence from network width.
We created FFNNs with ReLU activation function, each consisting of two hidden layers.
As shown in Figure~\ref{fig:nodesheat1}, the $y$-axis represents the number of nodes in the first hidden layer, while the $x$-axis represents the number of nodes in the second hidden layer. The values within the heatmap correspond to the validation accuracy of feedforward neural networks with ReLU activation function, trained for 10 epochs on the MNIST dataset, where the accuracy is determined based on the respective numbers of nodes in the $x$ and $y$ dimensions. 
The proposed method achieved a validation accuracy of $54.3\%$ when the number of nodes in each hidden layer was set to $2$. In contrast, other methods exhibited the lowest validation accuracy of only $10\%$. Our proposed method demonstrated independence from the number of nodes, effectively enabling the training of narrow feedforward neural networks with ReLU activation function.
We also recorded the validation accuracy at $1$ epoch of training on the FMNIST dataset to assess the convergence speed for various network sizes. 
The proposed method achieved the lowest validation accuracy of $25\%$ among the tested architectures when the first hidden layer had $2$ nodes, and the second hidden layer had $4$ nodes. In contrast, the lowest accuracy recorded by all other methods was $10\%$. In summary, our proposed method demonstrated independence from the number of nodes in FFNNs with ReLU activation function and converged faster compared to other methods.

\subsection{Activation Independent}\label{subsec:activation_indep}
Finally, we employed the proposed weight initialization method to assess its effectiveness in training feedforward neural networks, emphasizing its independence from the activation function. 
Table~\ref{table:activation} illustrates the validation accuracy of tanh, sigmoid, ReLU~\cite{relu1}, GeLU~\cite{gelu}, and SeLU~\cite{selu} on the MNIST and FMNIST datasets across 10 epochs. Here, GeLU and SeLU were set to their default settings in TensorFlow.
A feedforward neural network was constructed following the same configuration used in the layer independence experiment in Section~\ref{subsec:lay_indep}, comprising $120$ hidden layers, with a layer of $10$ nodes and a layer of $6$ nodes, repeated $60$ times each. Notably, for activation functions within the ReLU family, our method outperformed other weight initialization strategies. In particular, with the GeLU activation function, our proposed method achieved a validation accuracy of 68.1\% on FMNIST, and with ReLU, it reached an accuracy of 76.5\%, demonstrating higher accuracy compared to other weight initialization methods. Also, the proposed method showed high validation accuracy on the MNIST dataset. With the GeLU activation function, the proposed weight initialization method achieved an accuracy of 83.6\%, and with ReLU, it achieved an accuracy of 86.7\% (see Table~\ref{table:activation}).

\section{Conclusion}\label{sec:conclusion}
In this work, we propose a novel weight initialization method and provide several properties for the proposed initial weight matrix. 
We demonstrated the proposed matrix holds orthogonality.
Moreover, it was shown that 
the proposed initial matrix has
constant row(or column) sum. Also, we demonstrate that our weight initialization method ensures efficient signal transmission even in extremely deep and narrow feedforward ReLU neural networks.
Experimental results demonstrate that 
the network performs well regardless of whether it is 
deep or narrow, and even when there is a significant difference in the number of nodes between hidden layers.

\section*{Declaration of competing interest}
The authors declare that they have no known competing financial interests or personal relationships that could have appeared to influence the work reported in this paper.

\section*{Acknowledgement}
The authors wish to express their gratitude to the anonymous referees for their careful reading of the manuscript and their helpful suggestions. This work of H. Lee, Y. Kim, S. Yang, and H. Choi was supported by the National Research Foundation of Korea(NRF) grant funded by the Korea government(MSIT) (No. 2022R1A5A1033624). 
The work of Seung Yeop Yang was supported by Learning \& Academic research institution for Master’s · PhD students, and Postdocs(LAMP) Program of the National Research Foundation of Korea(NRF) grant funded by the Ministry of Education(No. RS2023-00301914).

\bigskip

\bibliographystyle{amsplain}

\end{document}